\documentclass[twoside]{article}

\usepackage[accepted]{aistats2024}
\usepackage[round]{natbib}

\usepackage{xcolor}         % colors
\usepackage{amsmath}
\usepackage{amsthm}
\usepackage{amssymb}
\usepackage{mathtools}
\usepackage{stmaryrd}

\usepackage{hyperref}

\usepackage[capitalize]{cleveref}
\usepackage{thmtools}
\usepackage{thm-restate}

\usepackage{enumitem}
\setlist[enumerate]{noitemsep, nolistsep}
\setlist[itemize]{noitemsep, nolistsep}

\usepackage{subcaption}
\usepackage{booktabs}       % professional-quality tables
\usepackage{tikz}
\usepackage{pgfplots}

\theoremstyle{plain}
\newtheorem{theorem}{Theorem}[section]

\newtheorem{lemma}[theorem]{Lemma}

\theoremstyle{definition}

\newtheorem{assumption}[theorem]{Assumption}
\newtheorem{example}[theorem]{Example}
\theoremstyle{remark}
\newtheorem{remark}[theorem]{Remark}

\definecolor{oxgreen}{RGB}{105,145,59}
\definecolor{oxred}{RGB}{190,15,52}
\definecolor{oxlighterblue}{HTML}{1a385a}

\usepackage{fancybox}
\newenvironment{fminipage}[1]%
    {\begin{Sbox}\begin{minipage}{#1}}%
  {\end{minipage}\end{Sbox}\fbox{\TheSbox}}

\Crefname{assumption}{Assumption}{Assumptions}

\newif\ifdraft
\draftfalse

\usepackage{import}
\usepackage{savesym}
\savesymbol{comment}

\ifdraft
\usepackage[draft]{commenting}
\else
\usepackage[nompar]{commenting}
\fi

\declareauthor{lo}{LO}{red} %teal
\authorcommand{lo}{comment}

\declareauthor{dw}{DW}{blue}
\authorcommand{dw}{comment}

\declareauthor{bk}{BK}{purple}
\authorcommand{bk}{comment}
\setdefaultauthor{dw}

\declareauthor{bk}{BK}{purple}
\authorcommand{bk}{comment}

\makeatletter
\renewcommand{\comm@todo@mpar}[1]{}
\makeatother

\def\divider{%
  \leavevmode\leaders\hrule height 0.6ex depth \dimexpr0.4pt-0.6ex\hfill%
  \kern0pt%
}

\renewcommand{\OpenCommentBraket}{$\lceil$}
\renewcommand{\ClosedCommentBraket}{$\rfloor$}

\newcommand\defeq{\coloneqq}

\newcommand\Real{\mathbb R}
\newcommand\Realp{\Real_{>0}}
\newcommand\Realnn{\Real_{\geq 0}}

\newcommand\compi\rightsquigarrow

\newcommand\csfc\rightsquigarrow

\newcommand\nat{\mathbb N}

\newcommand\sample[1]{\mathbf{sample}\,#1}

\newcommand\ifc[3]{\mathbf{if }\,#1 <0\,\mathbf{ then }\,#2\,\mathbf{ else }\,#3}

\newcommand\smooth{\sigma_\eta}

\newcommand\observe[2]{\mathbf{observe}\,#1\,\mathbf{ from }\,#2}

\newcommand\logpdf{\mathrm{logpdf}}

\newcommand\dist{\mathcal D}
\newcommand\mdist{\mathbf{\dist}}
\newcommand\nd{\mathcal N}

\newcommand\boundary\partial
\newcommand\interior\mathring

\newcommand\unif{\xrightarrow{\mathrm{unif.}}}

\newcommand\Jac{\mathbf{J}}

\newcommand{\tif}{\text{if }}

\newcommand{\tow}{\text{otherwise}}

\newcommand{\E}{\mathbb{E}}
 \DeclareMathOperator{\ELBO}{ELBO}

\newcommand*\diff{\mathop{}\!\mathrm{d}}

\newcommand\para{{\boldsymbol\theta}}
\newcommand\parasp{{\boldsymbol\Theta}}

\newcommand\lat\tr%{{\mathbf z}}

\newcommand\repx{\mathbf z}%{{\boldsymbol\alpha}}
\newcommand\rep{{\boldsymbol\phi}}

\newcommand\tr{\mathbf s}
\newcommand\from{:}
\newcommand\pto\rightharpoonup

\newcommand\pop{\mathrm{Op}}
\newcommand\sfct{\mathrm{Expr}}
\newcommand\sfcgs{\mathrm{SExpr}}%\sfct^{\mathrm{g.s.}}}
\newcommand\sfcae{\mathrm{SGuard}}%F^{\mathrm{a.e.}}}

\newcommand\relu{\mathrm{ReLU}}

\newcommand\sem[1]{\llbracket #1\rrbracket}
\newcommand\sema[1]{\llbracket #1\rrbracket_{\eta}}

\newcommand\lyy{\hbox{\textsc{Lyy18} }}

\definecolor{oxblue}{RGB}{0,33,71}
\newcommand\suppd{\mathbb S}

\begin{document}

\runningtitle{DSGD: Fast \& Convergent SGD for Non-Differentiable Models via Reparameterisation and Smoothing}


\twocolumn[
  \aistatstitle{Diagonalisation SGD: Fast \& Convergent SGD\\for Non-Differentiable Models\\via Reparameterisation and Smoothing}
  \aistatsauthor{ Dominik Wagner \And Basim Khajwal \And  C.-H. Luke Ong }
  \aistatsaddress{ NTU Singapore \And  Jane Street \And NTU Singapore } 
]

\begin{abstract}
  It is well-known that the reparameterisation gradient estimator, which exhibits low variance in practice, is biased for non-differentiable models. This may compromise correctness of gradient-based optimisation methods such as stochastic gradient descent (SGD).
  We introduce a simple syntactic framework to define non-differentiable functions piecewisely and present a systematic approach to obtain smoothings for which the reparameterisation gradient estimator is unbiased.
  Our main contribution is a novel variant of SGD, \emph{Diagonalisation Stochastic Gradient Descent}, which progressively enhances the accuracy of the smoothed approximation \emph{during} optimisation, and we prove convergence to stationary points of the \emph{unsmoothed} (original) objective.
  Our empirical evaluation reveals benefits over the state of the art: our approach is simple, fast, stable and attains orders of magnitude reduction in work-normalised variance.
\end{abstract}

% !TEX root = ../main.tex

\section{INTRODUCTION}

In this paper we investigate stochastic optimisation problems of the form
\begin{align}
  \label{eq:info}
  \text{argmin}_\para\,\E_{\repx\sim\mdist_\para}[f_\para(\repx)]
\end{align}
which have a wide array of applications, ranging from variational inference and reinforcement learning to queuing theory and portfolio design
\citep{MRFM20,BKM17,ZBKM19,SB98}.

We are interested in scenarios in which $f_\para$ is directly expressed in a programming language. Owing to the presence of if-statements, which arise naturally when modelling real-world problems (see \cref{sec:eval}), $f_\para$ may not be continuous, let alone differentiable.

In variational inference, Bayesian inference is framed as an optimisation problem and $f_\para$ is the \emph{evidence lower bound} (ELBO) $\log p(\mathbf x,\repx)-\log q_\para(\repx)$, where $p$ is the \emph{model} and $q_\para$ is a \emph{variational approximation}.
Our prime motivation is the advancement of variational inference for probabilistic programming\footnote{see e.g.\ \citep{MPYW18,BKS20} for introductions}---a new programming paradigm to pose and automatically solve Bayesian inference problems.

% Our ultimate goal is to target languages for constructing applications in numerical analysis and computational science such as Julia \cite{Julia-2017} %\footnote{\url{https://julialang.org/}}
% or in Bayesian statistics and probabilistic programming such as Stan~\cite{stan} and Pyro~\cite{pyro}.
% Also note that the distribution $q(\lat)$ (w.r.t.\ which the expectation is taken) is independent of the parameters $\theta$.

% Then $f$ is the so-called \emph{instantaneous ELBO}, $\log p(\mathbf x,\repx)-\log q_\para(\repx)$, and $p$ is called the \emph{model} and $q_\para$ the reparameterised \emph{guide}, a member of the variational family.

\paragraph{Gradient Based Optimisation.}
In practice the standard method to solve optimisation problems of the form \eqref{eq:info} are variants  of \emph{Stochastic Gradient Descent (SGD)}. Since the objective function is in general not convex, we cannot hope to always find global optima and
we seek \emph{stationary points} instead, where the gradient w.r.t.~the parameters $\para$ vanishes \citep{RM51}.

% Unfortunately, since $\E_{\lat\sim q(\lat)}[f(\para,\lat)]$ may not be convex, we cannot hope to always find global optima.
% We seek instead \emph{stationary points}, where the gradient w.r.t.~the parameters $\para$ vanishes.

A crucial ingredient for \emph{fast} convergence to a \emph{correct} stationary point is an estimator of gradients of the objective function $\E_{\repx\sim\mdist_\para}[f_\para(\repx)]$ which is both unbiased and has low variance.

% The \emph{log-derivative trick} can be used to derive an unbiased estimate, which is known as 
The \emph{Score} or \emph{REINFORCE} estimator \citep{RGB14,WW13,DBLP:conf/icml/MnihG14} makes little assumptions about $f_\para$ but it frequently suffers from high variance resulting in suboptimal results or slow/unstable convergence.

An alternative approach is the \emph{reparameterisation} or \emph{pathwise} gradient estimator.
The idea is to reparameterise the latent variables $\repx$ in terms of a known base distribution (entropy source) via a diffeomorphic transformation $\rep_\para$ (such as a location-scale transformation or cumulative distribution function).
E.g.~if $\mdist_\para(z)$ is a Gaussian $\mathcal{N}(z \mid \mu, \sigma^2)$ with $\para = \{\mu,\sigma^2\}$ then the location-scale transformation using the standard normal as the base gives rise to the reparameterisation
\[
z \sim \mathcal{N}(z \mid \mu, \sigma^2)
\iff
z = \mu + \sigma \, {s}, \quad s \sim \nd(0,1)
\]
In general, $\E_{\repx\sim \mdist_\para}[f_\para(\repx)]=\E_{\lat\sim\mdist}[f_\para(\rep_\para(\lat))]$ and its gradient can be estimated by \citep{MRFM20}:
\[
  \nabla_\para\,f_\para(\rep_\para(\lat))\qquad\lat\sim\mdist
\]
It is folklore that the reparameterisation estimator typically exhibits significantly lower variance in practice than the score estimator (e.g.\ \cite{MRFM20,DBLP:conf/icml/RezendeMW14,F06,SHWA15,XQKS19,LYY18}).
The reasons for this phenomenon are still poorly understood and examples do exist where score has lower variance than the reparameterisation estimator \citep{MRFM20}. %As a consequence, theoretical comparisons even between those two classic estimators are poorly understood and the subject of ongoing research. 

%\dw{Is the red for submission, or a workaround for a deficiency of the commenting package?} \lo{The latter. Now removed.}
% \lo{Good changes above - thanks!
% I also had it in mind to change the latent variable to another variable ($\not=$ $\lat$), so that it is clear that we use $z$ for the entropy source (which is to be reparameterised), $z \sim \nd(0,1)$, throughout.
% I think I remarked earlier that it was a bit confusing before.}
% See \cite{DBLP:journals/corr/KingmaW13,DBLP:conf/icml/TitsiasL14,DBLP:conf/icml/RezendeMW14}.
% \lo{In posterior inference, both KL-divergence and ELBO have the form $\E_{\lat\sim q_\para(\lat)}[f(\para,\lat)]$ where $f$ is $\log q_\para(\lat)-\log p(\lat \mid \mathbf x)$ for the former, and $\log p(\lat, \mathbf x) - \log q_\para(\lat)$ for the latter (see \href{https://www.cs.ox.ac.uk/teaching/materials20-21/SPP/Lectures/7-Variational.pdf}{Slide 21/38}).
% The key advantage of ELBO is the use of joint (as opposed to conditional) density.
% Using KL-divergence as our example could be confusing / attract criticisms.}
% \dw{Yes, you are of course right. I've changed it back.}

% \dw{reference}\lo{They would be the same references for the reparameterisation trick / gradient estimator. I think it unnecessary to repeat them here, as this is a well-known fact.}

% Due to the presence of conditionals in programming languages, $f$ is typically not continuous, let alone differentiable.\lo{repetition}
 Unfortunately, the reparameterisation gradient estimator
% \cite{LYY18,KW13,DBLP:conf/icml/TitsiasL14,RMW14}
is biased for non-differentiable models, which can be easily expressed in a programming language by if-statements:
\begin{example}
  \label{ex:biased}
  % \dw{This could be simplified further to just $[z\geq 0]$ (to make it easier for later in the paper and also to make it different from \cite{KOW23}). Slight disadvantage: true gradient is everywhere positive.}
  We recall a simple counterexample \cite[Prop.~2]{LYY18}. %and \cite[Ex.~1]{KOW23}:
  \begin{align*}
    \phi_\theta(s)&\defeq s+\theta
    &
    f(z)&= -0.5\cdot z^2+
    \begin{cases}
      0&\tif z<0\\
      1&\tow
    \end{cases}
  \end{align*}
  Observe that (see \cref{fig:smoothacc}):
    \begin{align*}
      \nabla_\theta\,\E_{s\sim\nd(0,1)}\left[f(\phi_\theta(z))\right]&=-\theta+\nd(-\theta\mid 0,1)\\
      &\neq -\theta =
      \E_{s\sim\nd(0,1)}\left[\nabla_\theta f(\phi_\theta(s))\right]
    \end{align*}
\end{example}
Employing a biased gradient estimator may \emph{compromise the correctness of stochastic optimisation}:
even if we can find a point where the gradient estimator vanishes, it may not be a critical point of the objective function \eqref{eq:info}. 
Consequently, in practice we may obtain noticeably inferior results 
%(see \cite[Fig.~1b]{KOW23} and 
(see our experiments later, \cref{fig:elbo}).

% Vanishing gradient estimator does \emph{not imply stationarity}!

\begin{figure}
  \centering
  \begin{subfigure}[t]{0.55\linewidth}
    \centering
    \begin{tikzpicture}
      \begin{axis}[
        width = 55mm,
        axis lines = middle,
        xlabel = \(\theta\),
        ]
        %Below the red parabola is defined
        \addplot [
        domain=-1:1,
        samples=100,
        color=oxred,
        % dashed,
        thick
        ]
        {-x};
        %    \addlegendentry{\(-\theta\)}
        %Here the blue parabola is defined
        \addplot [
        update limits=false,
        domain=-1:1,
        samples=100,
        color=oxgreen,
        thick
        ]
        {-x+1/sqrt(2*pi)*exp(-0.5*x*x)};
        \addplot[dotted,thick] table {figures/bias1.data};
        \addplot[dashed,thick] table {figures/bias2.data};

      \end{axis}
    \end{tikzpicture}
    \caption{Solid red: biased estimator $\E_{z\sim\nd(0,1)}\left[\nabla_\theta f(\theta,z)\right]$, solid green: true gradient $\nabla_\theta\,\E_{z\sim\nd(0,1)}\left[f(\theta,z)\right]$, black: gradient of smoothed objective (dotted: $\eta = 1$, dashed: $\eta = 1/3$) for \cref{ex:biased}.}
    \label{fig:smoothacc}
  \end{subfigure}
  \hfill
\begin{subfigure}[t]{0.4\linewidth}
  \centering
  \begin{tikzpicture}
    \begin{axis}[
      ymin=-0.1,
      ymax=1.1,
      width=45mm,
      axis lines = middle,
      ]
      \addplot [
      domain=-1:0,
      samples=100,
      color=oxred,
      line width=0.3mm
      ]
      {{0}};
      \addplot [
      domain=0:1,
      samples=100,
      color=oxred,
      line width=0.3mm
      ]
      {{1}};

      \addplot [
      domain=-1:1,
      samples=100,
      line width=0.3mm,
      dotted
      ]
      {{1/(1+exp(-3*x))}};

      \addplot [
      update limits=false,
      domain=-1:1,
      samples=100,
      line width=0.3mm,
      dashed
      ]
      {1/(1+exp(-15*x))};
    \end{axis}
  \end{tikzpicture}

  % \begin{tikzpicture}
  %   \begin{axis}[
  %     ymin=-0.1,
  %     ymax=3.95,
  %     height=4.5cm,
  %     axis lines = middle,
  %     ]

  %     \addplot [
  %     domain=-1:1,
  %     samples=100,
  %     line width=0.3mm,
  %     dotted
  %     ]
  %     {{3/(1+exp(-3*x))*1/(1+exp(3*x))}};

  %     \addplot [
  %     update limits=false,
  %     domain=-1:1,
  %     samples=100,
  %     line width=0.3mm,
  %     dashed
  %     ]
  %     {15/(1+exp(-15*x))*1/(1+exp(15*x))};
  %   \end{axis}
  % \end{tikzpicture}
\caption{Sigmoid function $\smooth$ (black dotted: $\eta=1/3$, black dashed: $\eta=1/15$) and the Heaviside step function (red, solid).}
\label{fig:sig}
\end{subfigure}
\label{fig}
    \caption{}
\end{figure}

\iffalse
\paragraph{Approach}
Our approach is based on an efficiently computable smoothed semantics of programs.
In particular, for \emph{accuracy coefficient} $k\in\nat$, we interpret if-statements $M\equiv\ifc{z+\theta}{ 0}{ 1}$ as
\begin{gather*}
  \sema M(\mathbf x)\defeq\sigma_k(z+\theta)
\end{gather*}
%where $\sigma_\eta(z)\defeq\sigma(\frac z\eta)=\frac 1{1+\exp(-\frac z\eta)}$.
with $\sigma_\eta(z)\defeq\sigma(\frac z\eta)$ where $\sigma(x) \defeq \frac 1{1+\exp(-x)}$ is the sigmoid function.
Thus, the program $M$ is interpreted by a smooth function $\sema M$, for which the reparameterisation gradient can be estimated unbiasedly.
  % , for which convergence results are well-established.
  Therefore, we propose to apply a stochastic gradient descent-like procedure which enhances the accuracy coefficient $k$ in each step.

\begin{compactitem}
  \item \emph{Smoothen} (discontinuous) function using sigmoid with accuracy coefficient
  \item Optimise expectation, enhancing accuracy in each step
\end{compactitem}
\fi

\iffalse
\lo{Hopefully Basim can exhibit empirically that running (DSGD') on the term \Cref{eq:counter eg} converges to the stationary point of the original optimisation problem, which is around $0.3722$, as $\nd(- 0.3722\mid 0,1) \approx 0.3722$.

Question: What happens when we run (SGD) on the LYY18 estimator of this term? Does it converge to the correct stationary point?
The point is that no proof exists that running (SGD) on the LYY18 estimator necessarily converges to a stationary point.
Do you think we might find a counterexample to the convergence of LYY18 estimator under (SGD)?

Reviewers of our paper will naturally compare our main contribution (provable convergence of (DDGS') to stationary points of an unbiased gradient estimator for typable programs) with the unbiased LYY18 estimator. The plus side: guaranteed convergence. Any negatives? Variance (work normalised) is still unknown.}
\fi

\paragraph{Systematic Smoothing.}

\cite{KOW23} present a smoothing approach to avoid the bias. To formalise the approach in a streamlined setting, we introduce a simple language to represent discontinuous functions piecewisely via if-statements, and we show how to systematically obtain a smoothed interpretation of such representations via sigmoid functions, which are parameterised by an \emph{accuracy coefficient}.
\iffalse
In particular, for \emph{accuracy coefficient} $\eta>0$, we interpret e.g.\ the if-statements in \cref{ex:biased} as
\begin{gather*}
  \sema {\ifc{z}{0}{1}}(z)\defeq \smooth(z)
\end{gather*}
%where $\sigma_\eta(z)\defeq\sigma(\frac z\eta)=\frac 1{1+\exp(-\frac z\eta)}$.
where $\smooth(x)\defeq\sigma(\frac x\eta)=\frac 1{1+\exp(-\frac x\eta)}$  (cf.~\cref{fig:sig}).
\fi

% \dw{outline DSGD}

\paragraph{Contributions.}
Our main contribution is the \emph{provable} correctness of a novel variant of SGD, \emph{Diagonalisation Stochastic Gradient Descent} (DSGD), to stationary points. The method takes gradient steps of \emph{smoothed} models whilst simultaneously enhancing the accuracy of the approximation in each iteration.
Crucially, asymptotic correctness is not affected by the choice of (accuracy) hyperparameters.

We identify mild conditions on our language and the distribution for theoretical guarantees.
In particular, for the smoothed problems we obtain unbiased gradient estimators, which converge uniformly to the true gradient as the accuracy is improved. Besides, as important ingredients for the correctness of DSGD we prove bounds on the variance, which solely depend on the syntactical structure of models.

Empirical studies show that DSGD performs comparably to the unbiased correction of the reparameterised gradient estimator by \cite{LYY18}.
However our estimator is simpler, faster, and attains orders of magnitude reduction in work-normalised variance.
Besides, DSGD exhibits more %robust behaviour
stable convergence
than using an optimisation procedure for \emph{fixed} accuracy coefficients \citep{KOW23}, which is heavily affected by the choice of that accuracy coefficient.

%\dw{comment on experiments}

\paragraph{Related Work.}
%\label{sec:relwork}

\cite{LYY18} is the starting point for our work and a natural source for comparison.
They correct the (biased) reparameterisation gradient estimator for non-differentiable models by additional non-trivial \emph{boundary} terms.
%Unfortunately, they only discuss efficient method for \emph{affine} guards.
They present an efficient (but non-trivial) method for \emph{affine} guards only.
Besides, they are not concerned with the \emph{convergence} of gradient-based optimisation procedures.
% ;
% nor do they discuss how the assumptions they make may be manifested in a programming language. \dw{affine restriction}
%\lo{The aim is point out the limitations of LYY18 objectively, without annoying them.}

\changed[dw]{
  \cite{MMT17,JGP17} study the reparameterisation gradient estimator and discontinuities arising from \emph{discrete} random variables, and they propose a continuous relaxation. This can be viewed as a special case of our setting since discrete random variables can be encoded via continuous random variables and if-statements.
  In the context of discrete random variables, \cite{TMMLS17} combine the score estimator with control variates (a common variance reduction technique) based on such continuous approximations.
  }

In practice, non-differentiable functions are often approximated smoothly. Some foundations of a similar smoothing approach are studied in
\cite[Thm.~3.1]{Z81} in a non-stochastic setting.

Abstractly, our diagonalisation approach resembles graduated optimisation \citep{BZ87,HLS16}: a ``hard'' problem is solved by ``simpler'' approximations in such a way that the quality of approximation improves over time.
However, the goals and merits of the approaches are incomparable: graduated optimisation is concerned with overcoming non-convexity of the objective function to find global optima rather than stationary points, whereas our approach is motivated by overcoming the bias of the gradient estimator of the objective function.

% \dw{BERTSEKAS: NONDIFFERENTIABLE OPTIMIZATION VIA
% APPROXIMATION*}

% \lo{Briefly discuss latest trends, and situate present work, in nonconvex and nonsmooth optimisation: state-of-the-art algorithms for finding critical points are not applicable to the present work because they typically assume Lipschitz continuity of the objective function.}

\cite{KOW23} study a (higher-order) probabilistic programming language and employ stochastic gradient descent on a fixed smooth approximation. 
Our DSGD algorithm advances their work in that it converges to stationary points of the \emph{original} (unsmoothed) problem.
Crucially, the accuracy coefficient does not need to be fixed in advance; rather it is progressively enhanced during the optimisation (which has important advantages such as higher robustness).

\iffalse
\begin{enumerate}
  \item We introduce the DSGD algorithm, and show that it converges to stationary points of the \emph{original} (unsmoothed) problem.
  Crucially, the accuracy coefficient does not need to be fixed in advance; rather it is %continuously enhanced (which has important advantages).
  \changed[lo]{progressively enhanced during the optimisation (which has important advantages).}
  \item We demonstrate that the \emph{gradients} of the smoothed objective functions converge uniformly.
  \item We bound the variance and prove (uniform) Lipschitz smoothness.
\end{enumerate}
\fi
% \lo{Repetition of points 2 and 3.}

\iffalse
Besides, our formalisation is significantly simpler and avoids jargon of programming language research, thus making systematic smoothing approaches for programming languages more accessible to a wider ML audience.
\fi

\iffalse
\paragraph{Outline}
%We start by reviewing related work (\cref{sec:relwork}). 
\cref{sec:setup} introduces our foundational framework, including a basic programming language for which we present a smoothed interpretation in \cref{sec:smooth}.
In \cref{sec:dsgd} we present our Diagonalisation Stochastic Gradient Descent procedure and its convergence, for which crucial pre-conditions are established in \cref{sec:precon}.
% \cref{sec:prob} clarifies connections to probabilistic programming.
We conduct an experimental evaluation in \cref{sec:eval} and conclude in \cref{sec:conc}.
\fi

% !TEX root = ../main.tex

\section{PROBLEM SETUP}
\label{sec:setup}
We start by introducing a simple function calculus to represent (discontinuous) functions in a piecewise manner.
% \lo{Trivial point, but I prefer ``in a piecewise manner''. ``Piecewisely'' is probably not an English word.} \dw{ok}

Let $\pop$ be a set of primitive functions/operations (which we restrict below) and $z_1,\ldots,z_n$ variables (for a fixed arity $n$).
% We define a class $\sfct$ of (syntactic) representation of piecewisely defined functions $\Real^n\to\Real$ inductively:
% \begin{align*}
%   \sfct\ni F::= z_j\mid f(F,\ldots,F)\mid \ifc FFF
% \end{align*}
% where $f\in\pop$. \dw{arity}
 We define a class $\sfct$ of (syntactic) representations of piecewisely defined functions $\Real^n\to\Real$ inductively:
\begin{align*}
  \sfct\ni F::=z_j\mid f(F,\ldots,F)\mid\ifc FFF
\end{align*}
where $f\in\pop$. That is: expressions are nestings of the following ingredients: variables, function applications and if-conditionals. To enhance readability we use post- and infix notation for standard operations such as $+,\cdot,^{2}$.
% \dw{
% Is this clear to ML people or is it better to write:
% \begin{enumerate}
%   \item variables $z_j$ are in $\sfct$
%   \item if $f:\Real^k\to\Real\in\pop$ and $F_1,\ldots,F_k\in\sfct$ then $f(F_1,\ldots,F_k)\in\sfct$
%   \item if $F,G,H\in\sfct$ then $\ifc FGH\in\sfct$
% \end{enumerate}
% } \lo{I think it is safe to use BNF.} \dw{ok}
\begin{example}
  \label{ex:terms}
  \cref{ex:biased} can be expressed as $F_1$:
  \begin{align*}
    F_1'&\equiv\ifc z0 1\\
    F_1&\equiv -0.5\cdot z^2+\ifc z0 1\\
    F_2&\equiv\mathbf{if}\, (a\cdot(\ifc{b\cdot z_1+c}01)\\
    &\qquad+d\cdot(\ifc{e\cdot z_2+f}01)+g)<0\\
    &\quad\;\mathbf{then}\, 0\,\mathbf{else}\,1
  \end{align*}
  $F_2$ illustrates that nested (if-)branching can occur not only in branches but also in the guard/condition.
  Such nestings arise in practice (see \texttt{xornet} in \cref{sec:eval}) and facilitates writing concise models.
\end{example}
$F\in\sfct$ naturally defines (see \cref{app:setup}) a function $\Real^n\to\Real$, which we denote by $\sem F$. 
In particular, for $f$ and $F_1$ defined in \cref{ex:biased,ex:terms}, respectively, $\sem {F_1}=f$.

% Let $\dist_\para$ be a distribution which is given by the reparameterisation:
% \begin{align*}
%   \repx\sim\mdist_\para\iff\repx=\rep_\para(\lat)\quad\lat\sim\mdist
% \end{align*}
% and let $q_\para$ and $q$ be the respective densities. Note that $q_\theta(\repx)=q(\rep_\para^{-1}(\repx))\cdot|\det\Jac\rep_\para(\repx)|$.

% Let $\dist$ be a distribution and $\rep_\para$ be a diffeomorphism for each $\para$. \dw{$\in\parasp$}

\paragraph{Expressivity.}
Our ultimate goal is to improve the inference engines of probabilistic programming languages such as Pyro \citep{pyro}, which is built on Python and PyTorch and primarily uses variational inference. 
A crucial ingredient of such inference engines are low-variance, unbiased gradient estimators. We identify if-statements (which break continuity and differentiability) as the key challenge.
%  Therefore, our language is designed for expressing piecewisely defined functions via if-statements. 
We deliberately omit most features of mainstream languages because they are not relevant for the essence of the challenge and would only make the presentation a lot more complicated. In practice more language features will be desirable and it is worthwhile future work to extend the supported language.

\paragraph{Problem Statement.}
We are ready to formally state the problem we are solving in the present paper:
\begin{center}
\noindent\begin{fminipage}{.95\columnwidth}
\[
  \text{argmin}_{\para\in\parasp} \E_{\lat\sim\mdist}[\sem F(\rep_\para(\lat))] \tag{\text{\textbf{P}}}\label{eq:optp}
\]
where $F\in\sfct$, $\dist$ is a continuous probability distribution with support $\suppd\subseteq\Real^n$, $\parasp\subseteq\Real^m$ is the \emph{parameter space} and each $\rep_\para:\suppd\to\suppd$ is a diffeomorphism\footnote{i.e.\ a bijective differentiable function with differentiable inverse}.
\end{fminipage}
\end{center}
\dw{foonote above good?}
Note that for
%\changed[dw]{$\mdist_\para$ with density $q_\para(\repx)\defeq q(\rep_\para^{-1}(\repx))\cdot|\det\Jac\rep_\para(\repx)|$}
$\mdist_\para(\repx)\defeq\mdist(\rep_\para^{-1}(\repx))\cdot|\det\Jac\rep_\para(\repx)|$,
\begin{align*}
  \E_{\repx\sim\mdist_\para}[\sem F(\repx)]=\E_{\lat\sim\mdist}[\sem F(\rep_\para(\lat))]
\end{align*}

Without further restrictions it is not \emph{a priori} clear that the optimisation problem is well-defined due to a potential failure of integrability.
Issues may be caused by both the distribution $\mdist$ and the expression $F$. E.g.\ the Cauchy distribution does not even have expectations, and despite $\nd(\theta,1)$---the normal distribution with mean $\theta$ and variance $1$---being very well behaved,
$\E_{z\sim\nd(\theta,1)}[\exp(z^2)]=\infty$
regardless of $\theta$.

\paragraph{Schwartz Functions.}
\label{sec:schw}
We slightly generalise the well-behaved class of Schwartz functions (see also for more details e.g.\ \citep{H15,R03}) to accommodate probability density functions the support of which is a subset of $\Real^n$:

% review some fundamental properties of the well-behaved class of Schwartz functions (see also for more details e.g.\ \citep{H15,R03}).

\changed[dw]{A function $f:\suppd\to\Real$, where $\suppd\subseteq\Real^n$ is measurable and has measure-0 boundary, is a (generalised) \emph{Schwartz function} if $f$ is smooth in the interior of $\suppd$ and for all $\alpha$ and $\beta$ (using standard multi-index notation for higher-order partial derivatives),
\begin{align*}
  \sup_{\mathbf x\in\interior\suppd}|\mathbf x^\beta\cdot\partial^\alpha f(\mathbf x)|<\infty
\end{align*}}
% We use $\Schwsp(\Real^n)$ for the space of such Schwartz functions.
Intuitively, a Schwartz function decreases rapidly.

% In this work we are particularly interested in Schwartz functions which are also the probability density function (pdf) of a continuous probability distribution.
% We call such functions \emph{Schwartz pdfs} and the corresponding distributions \emph{Schwartz distributions}. \dw{perhaps confusing with distribution theory}

\begin{example}
  % \cref{tab:dens} lists distributions with Schwartz pdfs. \dw{add gamma for fixed parameters}
  Distributions with pdfs which are also Schwartz functions include (for a fixed parameter)
  the (half) normal, exponential and logistic distributions. 
  %(see \cref{tab:dens}).
  % \dw{could add reference to appendix? (but space)}

  Non-examples include the Cauchy distribution and the Gamma distributions.
  % The Cauchy distribution famously does not even have an expectation (which would otherwise contradict \cref{lem:Schwprop} below). Even the pdf itself of the Gamma distribution for shape parameter $k<1$ is unbounded. 
  (The Gamma distribution cannot be reparameterised \citep{RTB16b} and therefore it is only of marginal interest for our work regarding the reparameterisation gradient.)
  % When the shape \dw{!!!} parameter $k,\alpha<1$ even the pdf itself is unbounded but
\end{example}
% We also focus on products of univariate Schwartz functions: if $f_1,\ldots,f_n:\Real\to\Real$ are (univariate) Schwartz functions then $f(x_1,\ldots,x_n)\defeq f_1(x_1)\cdot\cdots\cdot f_n(x_n)$ is a Schwartz function.

The following pleasing properties of Schwartz functions \citep{H15,R03} carry immediately over:
\begin{lemma}
  \label{lem:Schwprop}
  Let $f:\suppd\to\Real$ be a Schwartz function.
  \begin{enumerate}
    \item All partial derivatives of $f$ are Schwartz functions.
    \item\label{it:int} $f\in L^p(\suppd)$; in particular $f$ is integrable: $\int_\suppd |f(\mathbf x)|\diff\mathbf x<\infty$.
    \item\label{it:prod} The product $(f\cdot p):\suppd\to\Real$ is also a Schwartz function if $p:\Real^n\to\Real$ is a polynomial.
  \end{enumerate}
\end{lemma}

% \paragraph{Integrability.}
To mitigate the above (well-definedness) problem, 
%\lo{Which problem? It's unclear.} \dw{does this clarify sufficiently?} 
we henceforth assume:
\begin{assumption}
  \label{ass:basic}
  % Henceforth we assume that
  \begin{enumerate}
    \item The density $\mdist$ is a (generalised) Schwartz function on its support $\suppd\subseteq\Real^n$.
    \item $\pop$ is the set of {smooth} functions all partial derivatives of which are bounded by polynomials.
    \item \changed[dw]{$\rep_{(-)}$ and its partial derivatives are bounded by polynomials and each $\rep_\para:\suppd\to\suppd$ is a diffeomorphism.}
  \end{enumerate}
\end{assumption}

This set-up covers in particular typical variational inference problems (see \cref{sec:eval}) with normal distributions because log-densities $f(x)\defeq\log(\nd(x\mid\mu,\sigma^2))$ can be admitted as primitive operations.

Popular non-smooth functions such as $\relu,\max$ or the absolute value function are piecewise smooth. Hence, they can be expressed in our language using if-statements and smooth primitives. For instance, wherever a user may wish to use $\relu(x)$, this can be replaced with the expression $\ifc x 0 x$.

Employing \cref{lem:Schwprop} we conclude that the objective function in \eqref{eq:optp} is well defined: for all $\para\in\parasp$, $\E_{\lat\sim\mdist}[\sem F(\rep_\para(\lat))]<\infty$.

Whilst for this result it would have been sufficient to assume that just $\rep_{(-)}$ and $f\in\pop$ (and not necessarily their derivatives) are polynomially bounded, this will become useful later (\cref{sec:precon}).

\section{SMOOTHING}
\label{sec:smooth}
The bias of the reparametrisation gradient (cf.\ \cref{ex:biased}) is caused by discontinuities, which arise when interpreting if-statements in a standard way.
\cite{KOW23} instead avoid this problem by replacing the Heaviside step functions used in standard interpretations of if-statements with smooth approximations.

Formally, for $F\in\sfct$ and \emph{accuracy coefficient} $\eta>0$ we define the $\eta$-\emph{smoothing} $\sema F:\Real^n\to\Real$:
\begin{align*}
  \sema{z_j}(\repx)&\defeq z_j\\
  \sema{f(F_1,\ldots,F_k)}(\repx)&\defeq f(\sema{F_1}(\repx),\ldots,\sema{F_k}(\repx))\\
  \sema{\ifc FGH}(\repx)&\defeq\\
  \smooth(-\sema F(\repx))\cdot\sema G(\repx)
  &+\smooth(\sema F(\repx))\cdot\sema H(\repx)
\end{align*}
% Intuitively, we simply replace the Heaviside step-function arising in the standard interpretation of conditionals by
where $\smooth(x)\defeq\sigma(\frac x\eta)=\frac 1{1+\exp(-\frac x\eta)}$ is the logistic sigmoid function (see \cref{fig:sig}).

Note that the smoothing depends on the representation. In particular, $\sem F=\sem G$ does not necessarily imply $\sema F=\sema G$, e.g.\ $\sem{\ifc{z^2}{0}{z}}=\sem z$ but $\sema{\ifc{z^2}{0}{z}}\neq\sema z$.

% We note that the smoothed value function can be computed efficiently by (backward-mode) automatic differentiation
% \cite{Baydin2018}.

\paragraph{Unbiasedness and SGD for Fixed Accuracy Coefficient.}
\label{sec:unb}

Each $\sema F$ is clearly differentiable. Therefore, the following is a consequence of a well-known result about exchanging differentiation and integration, which relies on the dominated convergence theorem \citep[Theorem 6.28]{KOW23,K13}:
%
%
% This is relatively straightforward. We need to use the following:
% If $\parasp'\subseteq\parasp$ is a bounded set then the absolute value of polynomials $p:\parasp\times\Real^n$ are uniformly bounded by a polynomial $p':\Real^n$. We conclude:
\begin{restatable}[Unbiasedness]{proposition}{unb}
  \label{prop:unb}
  For every $\eta>0$ and $\para\in\parasp$,
  \begin{align*}
    \nabla_\para\,\E_{\lat\sim\mdist}[\sema F(\rep_\para(\lat))]=\E_{\lat\sim\mdist}[\nabla_\para\,\sema F(\rep_\para(\lat))]
  \end{align*}
\end{restatable}
Consequently, SGD can be employed on an $\eta$-smoothing for a fixed accuracy coefficient $\eta$.

% \subsection{SGD for Fixed Accuracy Coefficient}

% We apply SGD for fixed accuracy coefficient. To obtain theoretical convergence guarantees some conditions need to be checked. This can be done in the appendix as it's not very interesting and similar but more involved conditions need to be checked for DSGD.

\paragraph{Choice of Accuracy Coefficients.}

A natural question to ask is: \textit{how do we choose an accuracy coefficient such that SGD solves the original, unsmoothed problem \eqref{eq:optp} ``well''?}
For our running example we can observe that this really matters: stationary points for low accuracy (i.e.\ high $\eta$) may not yield significantly better results than the biased (standard) reparameterisation gradient estimator (see \cref{fig:smoothacc}).
% \lo{Check the slightly revised caption of \cref{fig:smoothacc}} \dw{add more data points}

On the other hand, there is unfortunately no bound (as $\eta\searrow 0$) to the derivative of $\smooth$ at $0$ (see \cref{fig:sig}). Therefore, the variance of the smoothed estimator also increases as the accuracy is enhanced.

% Consequently, we need to find a suitable tradeoff of the approximation error and the variance.
In the following section we offer a principled solution to this problem with strong theoretical guarantees.

% \cite{KOW23} provides a rudimentary analysis of the relationship between the $\eta$-smoothed and the original problem \eqref{eq:optp}: they show that under mild syntactic restrictions $\E_{\repx\sim\dist_\para}[\sema F(\repx)]$ converges uniformly to $\E_{\repx\sim\dist_\para}[\sem F(\repx)]$ as $\eta\searrow 0$.
% As a consequence, for any given error tolerance $\epsilon>0$, there \emph{exists} an accuracy coefficient $\eta$ such that the $\eta$-smoothed and true objective functions only differ by $\epsilon$.

% However, the practical relevance of this result is limited: it is generally not clear how to extract the accuracy coefficient from the convergence proof.
% Furthermore, we may wish to improve the accuracy of the approximation, e.g.\ once we have found a stationary point to an approximation that proves to be a relatively poor approximation to the true objective function.

% We offer a solution with theoretical guarantees for these issues in the following section.

% !TEX root = ../main.tex

\section{DIAGONALISATION SGD}
\label{sec:dsgd}

% So far, we have investigated the unbiasedness of smoothed gradient estimators and their use in gradient-based optimisation methods (conditional on their convergence to stationary points) for a \emph{fixed} accuracy coefficient.

%
% The objective of this section is twofold:
% Firstly, to obtain an (unconditional) convergence result.

We propose a novel variant of SGD in which we  enhance the accuracy coefficient \emph{during} optimisation (rather than \emph{fixing} it in advance). For an expression $F\in\sfct$ and a sequence $(\eta_k)_{k\in\nat}$ of accuracy coefficients we modify the standard SGD iteration to
\begin{align*}
  \para_{k+1}&\defeq\para_k-\gamma_k\nabla_\para \sem F_{\eta_k}(\para_k,\lat_k)&\lat_k\sim\mathcal D
\end{align*}
where $\gamma_k$ is the step size.
The qualifier ``diagonal'' highlights that, in contrast to standard SGD, we are not using the gradient of the same function $\sem F_{\eta}$ for each step but rather we are using the gradient of $\sem F_{\eta_k}$.
Intuitively, this scheme facilitates getting close to the optimum whilst the variance is low (but the approximation may be coarse) and make small adjustments once the accuracy has been enhanced and approximation errors become visible.

Whilst the modification to the algorithm is moderate, we will be able to \emph{provably guarantee} that asymptotically, the gradient of the original \emph{unsmoothed} objective function vanishes.

% \begin{enumerate*}
%   \item obtain (unconditional) convergence result
%   \item not fix accuracy coefficient but rather enhance it \emph{during} the optimisation.
% \end{enumerate*}

To formalise the correctness result, we generalise the setting: suppose for each $k\in\nat$, $f_k\from\parasp\times\Real^n\to\Real$ is differentiable.
We define a \emph{Diagonalisation Stochastic Gradient Descent} (DSGD) sequence:
\begin{align*}
  \para_{k+1}&\defeq\para_k-\gamma_k\nabla_\para f_k(\para_k,\lat_k)&\lat_k\sim\mathcal D
\end{align*}
Due to the aforementioned fact that also the variance increases as the accuracy is enhanced, the scheme of accuracy coefficients $(\eta_k)_{k\in\nat}$ needs to be adjusted carefully to tame the growth of the variances $V_k$ of the gradient of $f_k$, as stipulated by following equation:
\begin{align}
  \label{eq:stepvar}
  \sum_{k\in\nat}\gamma_k&=\infty & \sum_{k\in\nat}\gamma_k^2\cdot V_k<\infty
\end{align}
In the regime $f_k=f$ and $V_k=V$ of standard SGD, this condition subsumes the classic condition by \cite{RM51}, and $\gamma_k\in\Theta(1/k)$ is admissible.

The following exploits Taylor's theorem and can be obtained by modifying convergence proofs of standard SGD (see e.g.\ \citep{BT00} or \citep[Chapter~2]{B15}):
\begin{restatable}[Correctness]{proposition}{cdsgd}
  \label{prop:cdsgd}
  Suppose $(\gamma_k)_{k\in\nat}$ and $(V_k)_{k\in\nat}$ satisfy \cref{eq:stepvar},
  % \begin{align*}
  %   \sum_{k\in\nat}\gamma_k&=\infty&
  %   \sum_{k\in\nat}\gamma_k^2\cdot V_k&<\infty
  % \end{align*}
  % \begin{align*}
  %   g_k(\para)&\defeq\E_{\lat}[f_k(\para,\lat)]&g(\para)\defeq\E_{\lat}[f(\para,\lat)]
  % \end{align*}
  $g_k(\para)\defeq\E_{\lat}[f_k(\para,\lat)]$
   and $g(\para)\defeq\E_{\lat}[f(\para,\lat)]$
  are well-defined and differentiable,
   % and $g_k$ is differentiable.

  % Suppose there exist sequence $\{\para_k\mid k\in\nat\}\subseteq\parasp\subseteq\Real^m$,

  % Suppose for all $k\in\nat$ and $\para\in\parasp$,
  \begin{enumerate}[label=(D\arabic*)]
    \setlength{\itemindent}{3mm}
  % %  \setlength{\itemindent}{1.5em}
    \item\label{it:unb} $\nabla_\para g_k(\para)=\E_\lat[\nabla_\para f_k(\para,\lat)]$ for all $k\in\nat$, $\para\in\parasp$
    \item\label{it:lsmooth}\label{it:unibg}\label{it:bounded} $g$ is bounded, Lipschitz continuous and Lipschitz smooth\footnote{Recall that a differentiable function $g:\parasp\to\Real$ is $L$-Lipschitz smooth if for all $\para,\para'\in\parasp$, $\|\nabla g(\para)-\nabla g(\para')\|\leq L \|\para-\para'\|$.} on $\parasp$
    \item\label{it:bvar} $\E_{\lat}[\|\nabla_\para f_k(\para,\lat)\|^2]< V_k$ for all $k\in\nat$, $\para\in\parasp$
    \item\label{it:univcd} $\nabla g_k$ converges uniformly to $\nabla g$ on $\parasp$
    % \item\label{it:unibg} $\|\nabla_\para\E_{\lat}[f(\para,\lat)]\|< L$
  \end{enumerate}
  Then
  % \emph{$\inf_{i\in\nat}\E[\|\nabla g(\para_i)\|^2]=0$} or $\para_i\not\in\parasp$ for some $i\in\nat$.
  almost surely\footnote{w.r.t.\ the random choices $\lat_1,\lat_2,\ldots$ of DSGD} $\liminf_{i\to\infty}\|\nabla g(\para_i)\|=0$ or $\para_i\not\in\parasp$ for some $i\in\nat$.
\end{restatable}

\iffalse
\begin{remark}%{remark}{str}
  \label{rem:str}
  If the step sizes and gradients converge sufficiently quickly, formally $\sum_{k\in\nat}\gamma_k^2<\infty$ and $\sum_{k\in\nat} \gamma_k\|\nabla g_k-\nabla g\|<\infty$, the former disjunct can be strengthened to $\lim_{i\to\infty}\|\nabla g(\para_i)\|=0$, which implies that accummulation points are stationary.
\end{remark}
\fi

%\lo{Clarify the distribution w.r.t.~the expectation $\E[\|\nabla g(\para_i)\|^2]$ is taken.} \dw{is this sufficient?} \lo{Yes. Thanks.}
% \dw{differentiability of $g$?}
% \dw{intuition of assumptions}

% \dw{boundedness!!!}

% \section{Application}
% \label{sec:appl}

Having already discussed unbiasedness \ref{it:unb}, \cref{prop:unb}, we address the remaining premises in the next section to show that DSGD is correct for expressions.
 % we establish the remaining pre-conditions for the instantiation $f_k\defeq \llbracket F\rrbracket_{\eta_k}\circ\rep_{(-)}$. Recall that .

% Having established \ref{it:unb} in \cref{prop:unb} and \ref{it:univcd} in \cref{thm:unicon}, the main remaining properties we need to establish are (uniform) Lipschitz smoothness for \ref{it:lsmooth} and a bound on the variance \ref{it:bvar}.

\section{ESTABLISHING PRE-CONDITIONS}
\label{sec:precon}

% \subsection{Schwartz Functions on Open Sets}

Note that the pre-conditions of \cref{prop:cdsgd} may fail for non-compact $\parasp$: e.g.\ the objective function $\theta\mapsto\E_{s\sim\nd(0,1)}[s+\theta]$ is unbounded.
Therefore, we assume the following henceforth:
% \footnote{Intuitively, otherwise arbitrary much probability mass  may be placed close to the discontinuities, compromising e.g.\ uniform convergence (see~\cref{ex:comp} in \cref{app:precon}).}:
\begin{assumption}
  \label{ass:comp}
 $\parasp\subseteq\Real^m$ is compact.
\end{assumption}

% \subsection{Motivation: Lipschitz Smoothness}
% \label{sec:lipsch}

% Recall our our running example $\ifc z 0 1$ and $\dist=\nd(0,1)$, $\phi_\theta(s)=s+\theta$ and $\dist_\theta=\nd(\theta,1)$. Note that
% \begin{align*}
%   &\left|\frac{\partial^2}{\partial\theta^2}\E_{\lat\sim\dist}[\sem{\ifc z 0 1}(\phi_\theta(s))]\right|\\
%   &=\left|\frac{\partial^2}{\partial\theta^2}\E_{\repx\sim\dist_\theta}[\sem{\ifc z 0 1}(\repx)]\right|\\
%   &=\left|\int_0^\infty \frac {\partial^2 q_{(-)}}{\partial\theta^2}(\theta,z)\diff z\right|\\
%   &\leq\int\left| \frac {\partial^2 q_{(-)}}{\partial\theta^2}(\theta,z)\right|\diff z\\
%   &\leq\int\left| q_\theta(z)+\theta^2\cdot q_\theta(z)\right|\diff z\leq 1+\theta^2
% \end{align*}
% \dw{check}

For Lipschitz continuity it suffices to bound the partial derivatives of the objective function. Thus, we exchange differentiation and integration\footnote{This is valid because $\mdist_\para$ is differentiable and $\sem F$ is independent of $\theta_i$.}:
% \begin{align*}
  \[
    \left|\frac{\partial}{\partial\theta_i}\E_{\repx\sim\mdist_\para}[\sem F(\repx)]\right|
  \leq
  \int \left|\sup_{\para\in\parasp}\frac{\partial}{\partial\theta_i}\mdist_\para(\repx)\right|\cdot\left|\sem F(\repx)\right|\diff\repx
  \]
% \end{align*}
Extending the integrability result for Schwartz functions and polynomials (\cref{lem:Schwprop}) in a non-trivial way (see \cref{app:precon}), we can demonstrate that the integral on the right side is finite.
\changed[dw]{
Our proof (cf.~\cref{app:unicon}) relies on the following:
\begin{assumption}
  \label{eq:strdiff}
  $\rep_{(-)}:\parasp\times\suppd\to\suppd$ satisfies
  %  is a \emph{strong} diffeomorphism, by which we mean
    \begin{align*}
      \inf_{(\para,\lat)\in\parasp\times\suppd}|\det\Jac\rep_\para(\lat)|>0
    \end{align*}
\end{assumption}  
This requirement is a bit stronger than $|\det\Jac\rep_\para(\lat)|>0$ for each $(\para,\lat)\in\parasp\times\suppd$, which automatically holds if each $\rep_\para:\suppd\to\suppd$ is a diffeomorphism. 

% \footnote{A (contrived) counterexample is $\phi(z_1,z_2)\defeq z_1^2z_2+\frac 1 3 z_1^2 z_2^3-z_1z_2^2+z_2$ because $\frac{\partial\phi}{\partial z_2}(z_1,z_2)=z_1^2+(z_1z_2-1)^2$ has no positive lower bound despite being everywhere positive.}

Our prime examples, location-scale transformations, satisfiy this stronger property:
\begin{example}
  Suppose $\parasp\subseteq\Real^m$ is compact, $f:\parasp\to\Real_{>0}$ is continuous and $g:\parasp\to\Real$. Then the location-scale transformation
  $\phi_\para(s)\defeq f(\para)\cdot s+g(\para)$
  satisfies \cref{eq:strdiff} because $\inf_{\para\in\parasp} f(\para)>0$ ($\parasp$ is compact and $f$ is continuous).
\end{example}
}

Similarly, we can prove the other obligations of \ref{it:bounded}.

\subsection{Uniform Convergence of Gradients}
\label{sec:unicon}

\cite{KOW23} show that under mild conditions, the smoothed objective function converges uniformly to the original, unsmoothed objective function. For \ref{it:univcd} we need to extend the result to gradients.

%Streamlining the complicated type system of \cite{KOW23},
% \subsection{Syntactic Restrictions for Convergence}

% \dw{simple example which doesn't converge}
% As a first step, we aim to ensure that $\sema F$ converges to $\sem F$ almost everywhere.
Recall that $\smooth$ converges (pointwisely) to the Heaviside function on $\Real\setminus\{0\}$ (cf.\ \cref{fig:sig}).
On the other hand, $\E_{z\sim\nd(\theta,1)}[\sema{\ifc 001}(z)]=\frac 1 2$ converges \emph{nowhere} to $\E_{z\sim\nd(\theta,1)}[\sem{\ifc 001}(z)]=1$.

To rule out such contrived examples, we require that conditions in if-statements only use functions which are a.e.\ not $0$.
 % \dw{does this make sense?} \lo{yes}
%\lo{How about: guards are not the constant 0}
% \dw{That's not quite enough. Something like $\ifc{(\ifc z 0 1)}01$ also needs to be ruled out as $\ifc z 0 1$ is $0$ on a non-null set.}
Formally, we define \emph{safe} guards and expressions inductively by:
\begin{gather*}
    \sfcae\ni G ::= f(z_{i_1},\ldots, z_{i_k})\mid\ifc GGG\\
    \sfcgs\ni S ::= z_j \mid f(S,\ldots,S)\mid\ifc G SS
\end{gather*}
where in the first rule we assume that $f\neq 0$ a.e.\footnote{By \cite{M15} this can be guaranteed if $f$ is analytic, not constantly $0$. (Recall that a function
$f: R^m\to R^n$ is analytic if it is infinitely differentiable and its multivariate Taylor expansion at every point $x_0\in\Real^m$ converges pointwise to $f$ in a
neighbourhood of $x_0$.)} and the $i_j$ are pairwise distinct.

Note that $\sfcae\subset\sfcgs\subset\sfct$ and for $G\in\sfcae$, $\sem G\neq 0$ a.e.
As a consequence, by structural induction, we can show that for $S\in\sfcgs$, $\sema S$ converges to $\sem S$ almost everywhere.

% In this section we expand on this result by showing that also the \emph{gradients} converge uniformly.

Exploiting a.e.\ convergence we conclude not only the uniform convergence of the smoothed objective function but also their \emph{gradients}:
% the fact that both $q_\para$ and $\frac\partial{\partial\theta_i}q_\para$ are generalised Schwartz functions we can show:
\begin{restatable}[Uniform Convergence]{proposition}{unicon}
  \label{thm:unicon}
  If $F\in\sfcgs$ then
   \begin{align*}
     \E_{\repx\sim\mdist_\para}[\sema F(\repx)]&\unif\E_{\repx\sim\mdist_\para}[\sem F(\repx)]\\
     \nabla_\para\,\E_{\repx\sim\mdist_\para}[\sema F(\repx)]&\unif\nabla_\para\,\E_{\repx\sim\mdist_\para}[\sem F(\repx)]
   \end{align*}
   as $\eta\searrow 0$ for $\para\in\parasp$.
\end{restatable}
% We outline the main technical ingredients of the proof in the next subsection.

\subsection{Bounding the Variance}
\label{sec:var}

Next, we analyse the variance for \ref{it:bvar}.
Recall that nesting if-statements in guards (e.g.\ $F_2$ in \cref{ex:terms}) results in nestings of $\smooth$ in the smoothed interpretation, which in view of the chain rule may cause high variance as $\smooth'(0)=\frac 1{4\eta}$.
% \begin{align*}
%   \frac\partial{\partial\theta_i}\smooth(\smooth(\rep_\para(\lat))=\smooth'(\smooth(\rep_\para(\lat)))\cdot\smooth'(\rep_\para(\lat))\cdot\frac{\partial}{\partial\theta_i}\rep_\para(\lat)
% \end{align*}
Therefore, to give good bounds, we classify expressions by their maximal nesting depth $\ell$ of the conditions of if-statements.
Formally, we define $\sfct_{\ell}$ inductively
\begin{enumerate}
  \item $z_j\in\sfct_0$
  \item If $f:\Real^k\to\Real\in\pop$ and $F_1,\ldots,F_k\in\sfct_\ell$ then $f(F_1,\ldots,F_k)\in\sfct_\ell$
  \item If $F\in\sfct_\ell$ and $G,H\in\sfct_{\ell+1}$ then\\$(\ifc FGH)\in\sfct_{\ell+1}$
  \item If $F\in\sfct_\ell$ then $F\in\sfct_{\ell+1}$
\end{enumerate}
Note that $\sfct=\bigcup_{\ell\in\nat}\sfct_\ell$. For the expressions in \cref{ex:terms}, $F_1,F'_1\in\sfct_1$ and $F_2\in\sfct_2$.

Now, exploiting the chain rule and the fact that $|\smooth'|\leq\eta^{-1}$,
it is relatively straightforward to show inductively that for $F\in\sfct_\ell$ there exists $L>0$ such that $\E_{\lat\sim\mdist}[\|\nabla_\para\sema F(\rep_\para(\lat))\|^2]\leq L\cdot\eta^{-2\ell}$ for all $\eta>0$ and $\para\in\parasp$.
However, we can give a sharper bound, which will allows us to enhance the accuracy more rapidly (in view of \cref{eq:stepvar}):
\begin{restatable}{proposition}{var}
  \label{prop:var}
  If $F\in\sfct_\ell$ then there exists $L>0$ such that for all $\eta>0$ and $\para\in\parasp$, $\E_{\lat\sim\mdist}[\|\nabla_\para\sema F(\rep_\para(\lat))\|^2]\leq L\cdot\eta^{-\ell}$.
\end{restatable}
To get an intuition of the proof presented in \cref{app:var}, we consider our running \cref{ex:terms}, $F'_1\equiv(\ifc z 0 1)\in\sfct_1$ and $\phi_\theta(s)=s+\theta$:
{\allowdisplaybreaks
\begin{align*}
  &\E_{s\sim\nd}\left[\left|\frac{\partial(\sema {F'_1}\circ\phi_{(-)})}{\partial\theta}(\theta,s)\right|^2\right]=\E\left[\left(\smooth'(\phi_\theta(s))\right)^2\right]\\
  &\leq\eta^{-1}\cdot\int\nd(s)\cdot\frac{\partial(\smooth\circ\phi_{(-)})}{\partial s}(\theta,s)\diff s\\
  &=\eta^{-1}\cdot\left(\left[\nd(s)\cdot\smooth(\phi_\theta(s))\right]_{-\infty}^\infty-\int \nd'(s)\cdot\smooth(\phi_\theta(s))\diff s\right)\\
  &\leq\eta^{-1}\cdot\int |\underbrace{\nd'(s)}_{\text{Schwartz}}|\diff s\leq\eta^{-1}\cdot L
\end{align*}
}
where we used integration by parts in the third step.

\subsection{Concluding Correctness}
\label{sec:cc}
Having bounded the variance, we present a scheme of accuracy coefficients compatible with the scheme of step sizes $\gamma_k=1/k$, which is the classic choice for SGD.
Note that for any $\epsilon>0$, $\sum_{k\in\nat}\frac 1 {k^2}\cdot (k^{\frac 1\ell-\epsilon})^\ell<\infty$.
% \begin{align*}
%   \sum_{k\in\nat}\frac 1 k&=\infty&
%   \sum_{k\in\nat}\frac 1 {k^2}\cdot \left(k^{\frac 1\ell-\epsilon}\right)^\ell<\infty
% \end{align*}
Therefore, by \cref{prop:var} we can choose accuracy coefficients $\eta_k\in\Theta(k^{-\frac 1\ell+\epsilon})$. 
%(NB The sharper bound allows us to enhance the accuracy more rapidly.)
Finally, with \cref{prop:cdsgd,thm:unicon} we conclude the correctness of DSGD for smoothings:
\begin{theorem}[Correctness of DSGD]
  \label{thm:corr}
  Let $F\in\sfct_\ell\cap\sfcgs$ and $\epsilon>0$. 
  Then DSGD is correct for $\sem F_{\eta_k}$,
  %  (in the sense of the conclusion of \cref{prop:cdsgd})
    %for $f_k\defeq\llbracket F\rrbracket_{\eta_k}\circ\rep_{(-)}$, 
    $\gamma_k\in\Theta(1/k)$ and $\eta_k\in\Theta(k^{-\frac 1\ell+\epsilon})$:\\
  almost surely $\liminf_{i\to\infty}\|\nabla_\para\,\E_{\lat\sim\mdist}[\sem F(\rep_\para(\lat))]\|=0$ or $\para_i\not\in\parasp$ for some $i\in\nat$.
  % \lo{Correctness here means that the preconditions of \cref{prop:cdsgd} hold with $f_k = \llbracket F\rrbracket_{\eta_k}\circ\rep_{(-)}$, and hence also the conclusion.}
\end{theorem}
For instance for $F\in\sfct_1$ we can choose $\eta_k\in\Theta(1/{\sqrt k})$.
Crucially, for the choice of accuracy coefficients only the syntactic structure (i.e.\ nesting depth) of terms is essential. In particular, there is no need to calculate bounds on the Hessian, the constant in \cref{prop:var}, etc.

% \dw{Great theoretical advantage: regardless of constants, convergence asymptotically}

% \paragraph{Theoretical Comparison of Variances.}

\changed[dw]{
\subsubsection{Discussion}

\paragraph{Accuracy coefficient schedule for other learning rate schemes.}

% \textit{Section 5.3: Is there a recommended way to derive scheduling for other learning rate schemes? For example, cosine annealing is often used in large-scale deep learning. Can we derive a schedule for that?}

The requirement on the step sizes and accuracy coefficients stipulated by \cref{eq:stepvar} can be relaxed to
\begin{equation}
    \label{eq:var}
    (\sum_{k\in\mathbb N}\gamma_k)^{-1}\cdot\sum_{k\in\mathbb N}\gamma_k^2\cdot V_k=0
\end{equation}
 which is useful for deriving other admissible accuracy coefficient schemes.
For instance, step size $\gamma_k=1/\sqrt k$ violates \cref{eq:stepvar}. However, for terms with nesting depth $\ell=1$, the accuracy coefficients $\eta_k=k^{-0.5+\epsilon}$ (where $\epsilon>0$ is arbitrary) satisfy this relaxed requirement, and we obtain the same correctness guarantees as Theorem 5.4.

  \paragraph{Choice of $\epsilon$.}
  % \begin{remark}[Choice of $\epsilon$]
    Whilst asymptotically \emph{any} choice of $\epsilon>0$ enjoys the theoretical guarantees of \cref{thm:corr}, the \emph{speed} of convergence of the gradient norm of the unsmoothed objective is governed by two summands:
    (1) the speed of convergence of the gradient of the smoothed objective function to the gradient of the unsmoothed objective, and
    (2) the accumulated contribution of the variances (more formally: the speed of convergence of the finite sums of the relaxation \cref{eq:var} above of \cref{eq:stepvar} to $0$ and thus the magnitude of $\sum_{i\in\mathbb N}1/k^2\cdot (k^{1/\ell-\epsilon})^\ell$).
    
    % the accumulated contribution of the variances (more formally: the magnitude of $\sum_{i\in\mathbb N}\frac 1{k^2}\cdot (k^{\frac 1\ell-\epsilon})^\ell$, which determines the speed of convergence of the finite sums in \cref{eq:dsgdproof} in \cref{app:dsgd} to $0$).
Whilst for the former, small $\epsilon>0$ is beneficial (enhancing the accuracy rapidly), for the latter large $\epsilon$ is beneficial. Consequently, for best performance in practice a trade-off between the two effects needs to be made, and we found the choices reported below to work well.
}
  % Whilst asymptotically the selection of $\epsilon>0$ does not matter in theory, in practice the choice has an effect on performance in a (necessarily finite) optimisation trajectory. Empirically, we found the values reported in \cref{sec:expsetup} to work well.

  % The choice of $\epsilon$ needs to make a trade-off between two phenomena: smaller $\epsilon$ means faster convergence of approximation quality, whilst larger $\epsilon$ causes the (finite) variance term in \cref{eq:morrmon2} to be smaller, which in view of the proof of \cref{prop:cdsgd} is beneficial for small norms of gradients.
% \end{remark}

% \begin{restatable}[Average Variance of Run]{remark}{avvar}
%   \label{rem:avvar}
\paragraph{Average Variance of Run}
  By \cref{prop:var} the average variance of a finite DSGD run with length $N$ is bounded by (using H\"older's inequality)
  \[
  \frac 1N\cdot\sum_{k=1}^N\eta_k^{-\ell} \leq \eta_{\frac{N+1}2}^{-\ell}
  \]
  Consequently, the average variance of a DSGD run is lower than for standard SGD with a fixed accuracy coefficient $\eta<\eta_{\frac{N+1}2}$, the accuracy coefficient of DSGD after half the iterations.

\section{EMPIRICAL EVALUATION}
% !TEX root = ../main.tex

\label{sec:eval}

We evaluate our \textsc{DSGD} procedure against SGD with the following gradient estimators:
the biased reparameterisation estimator (\textsc{Reparam}), the unbiased correction thereof (\textsc{LYY18}, \citep{LYY18}), the smoothed reparametrisation estimator of \cite{KOW23} for a \emph{fixed} accuracy coefficient (\textsc{Fixed}), and the unbiased (\textsc{Score}) estimator.

\paragraph{Models.} We include the models from \cite{LYY18,KOW23} and add a \texttt{random-walk} model. We summarise some details (the additional models are covered in \cref{app:eval}):
% \begin{itemize}
%   \item 

\texttt{temperature} \cite{DBLP:conf/qest/SoudjaniMN17} model a controller keeping the temperature of a room within set
bounds. The discontinuity arises from the discrete state of the controller, being either on or off. The model has a 41-dimensional latent variable and 80 if-statements.
  % \item 

  \texttt{random-walk} models a random walk (similar to \citep{MOPW}) of bounded length. The goal is to infer the starting position based on the distance walked. The walk stops as soon as the destination is reached. This is checked using if-statements and causes discontinuities. In each step a normal-distributed step is sampled and its absolute value is added to the distance walked so far, which accounts for more non-differentiabilities. Overall, the model has a 16-dimensional latent variable and 31 if-statements.
  
  \texttt{xornet} is a multi-layer neural network trained to compute the XOR function with all activation functions being the Heaviside step function.
  %  which is traditionally infeasible for gradient-based optimisation because of the discontinuity at $0$ and a zero gradient everywhere else.
  The model has a 25-dimensional latent space (for all the weights and biases) and 28 if-statements.
  The \lyy estimator is \emph{not} applicable to this model since the branch conditions are not all affine in the latent space.
  
  $F_2$ from \cref{ex:terms} can be viewed as a (stark) simplification of \texttt{xornet}. It is also the only model in which the guards of if-statements contain variables which in turn depend on branching. As such, \texttt{xornet} corresponds to a term in $\sfct_3$, whereas all other models correspond to $\sfct_1$.

  % \item

% \end{itemize}

\begin{figure}[t!]

  \vspace*{-3mm}
  \centering
  %\small
  \begin{subfigure}[h]{0.8\linewidth}
    \includegraphics[width=\linewidth]{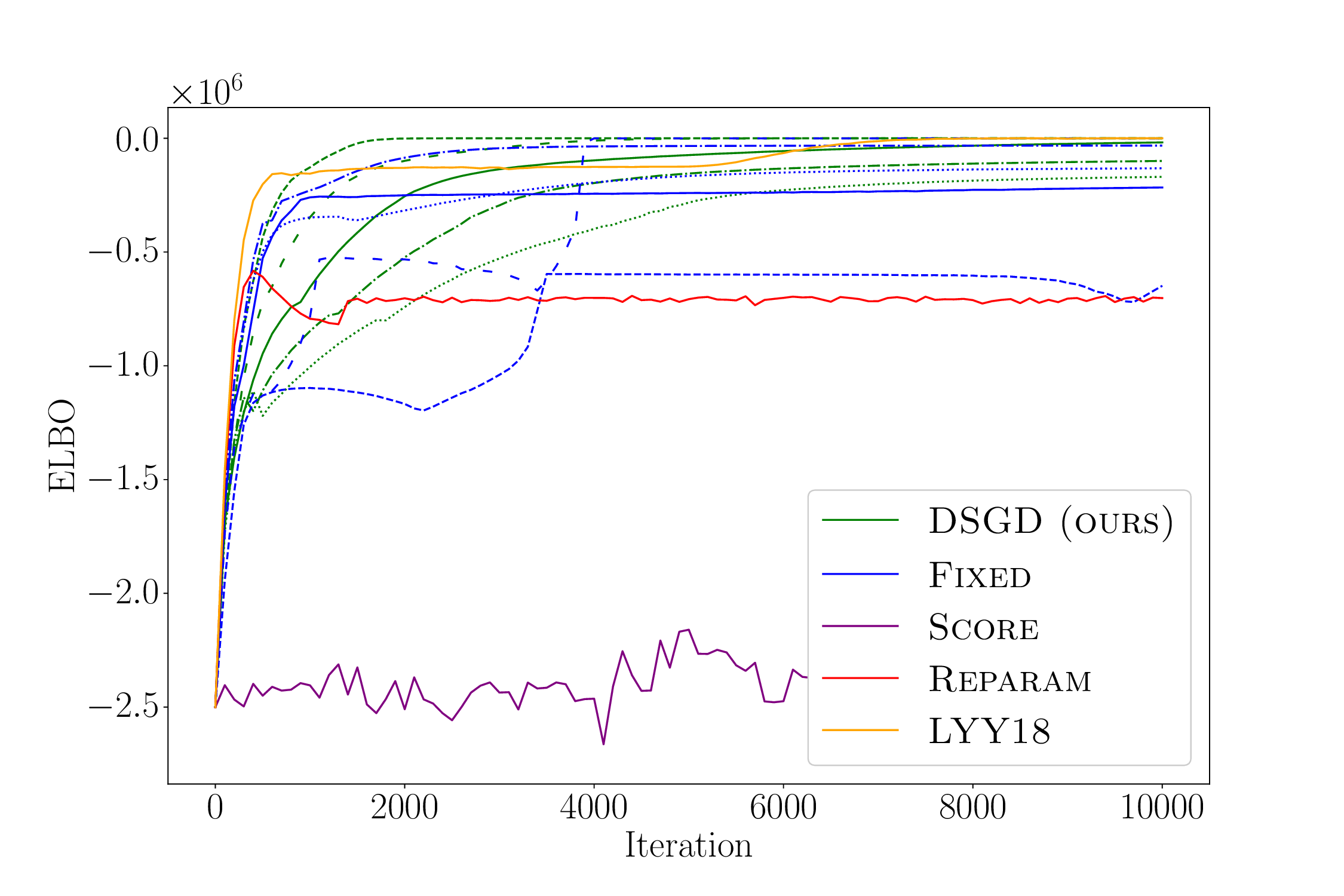}

    \vspace*{-3mm}
    \caption{\texttt{temperature}}
    \label{fig:temperature-graph}
  \end{subfigure}
  \\[-5pt]
  \begin{subfigure}[h]{0.8\linewidth}
    \includegraphics[width=\linewidth]{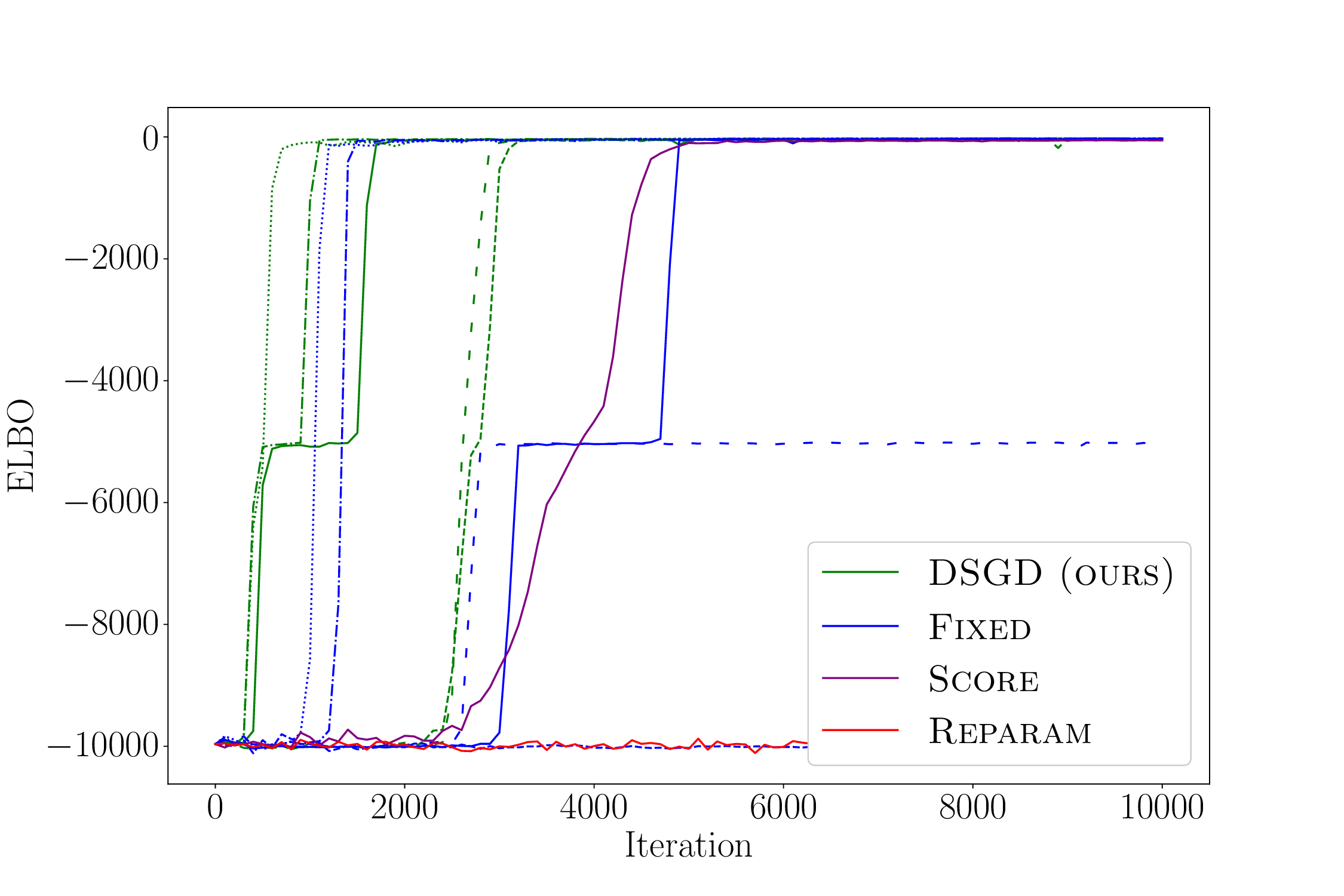}

    \vspace*{-3mm}
    \caption{\texttt{xornet}}
    \label{fig:xornet-graph}
  \end{subfigure}
\\[-5pt]
% \vspace*{0pt}
\begin{subfigure}[h]{0.8\linewidth}
  \includegraphics[width=\linewidth]{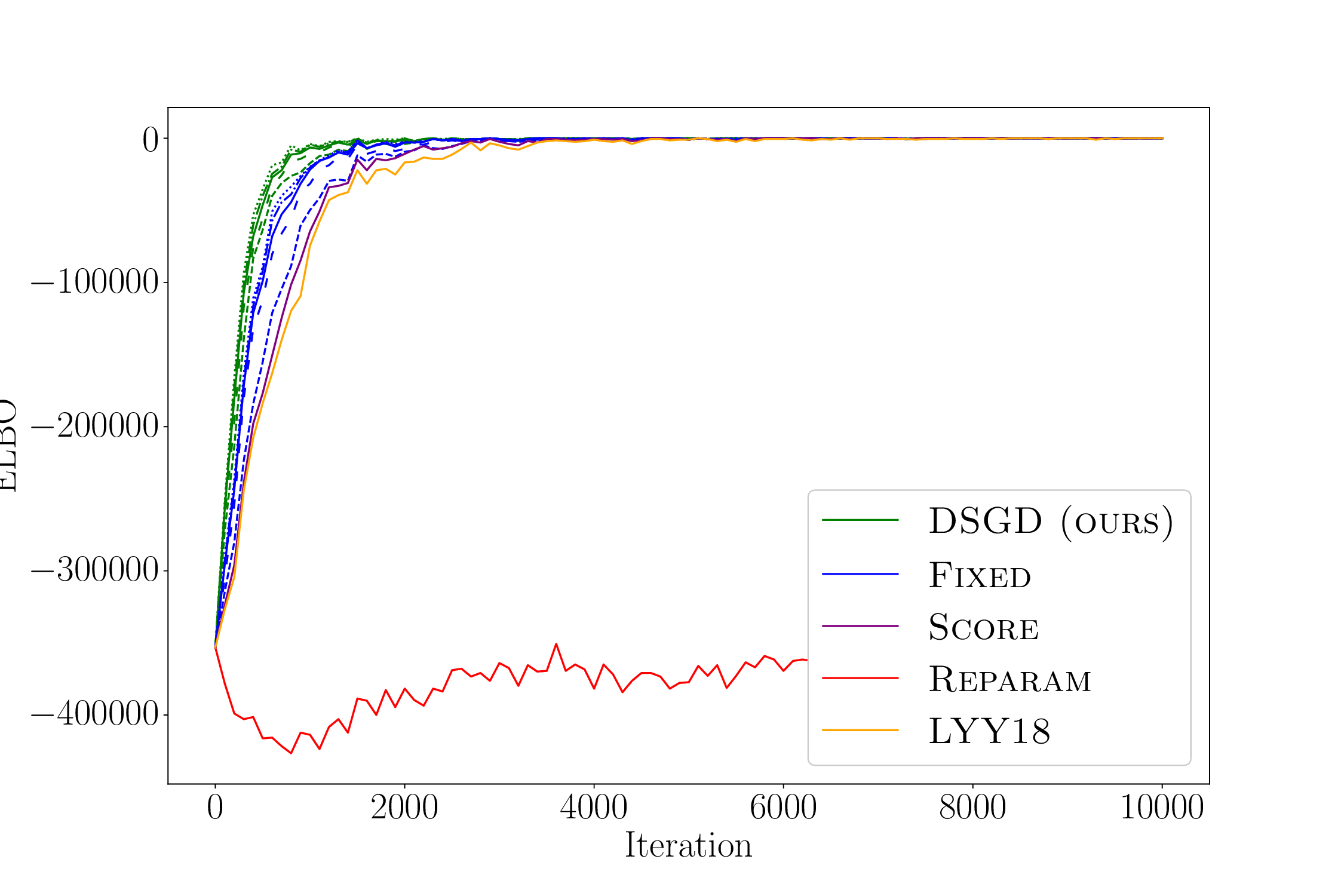}

  \vspace*{-3mm}
  \caption{\texttt{random-walk}}
  \label{fig:walk-graph}
\end{subfigure}
\\[-5pt]
\begin{subfigure}[h]{0.8\linewidth}
  \includegraphics[width=\linewidth]{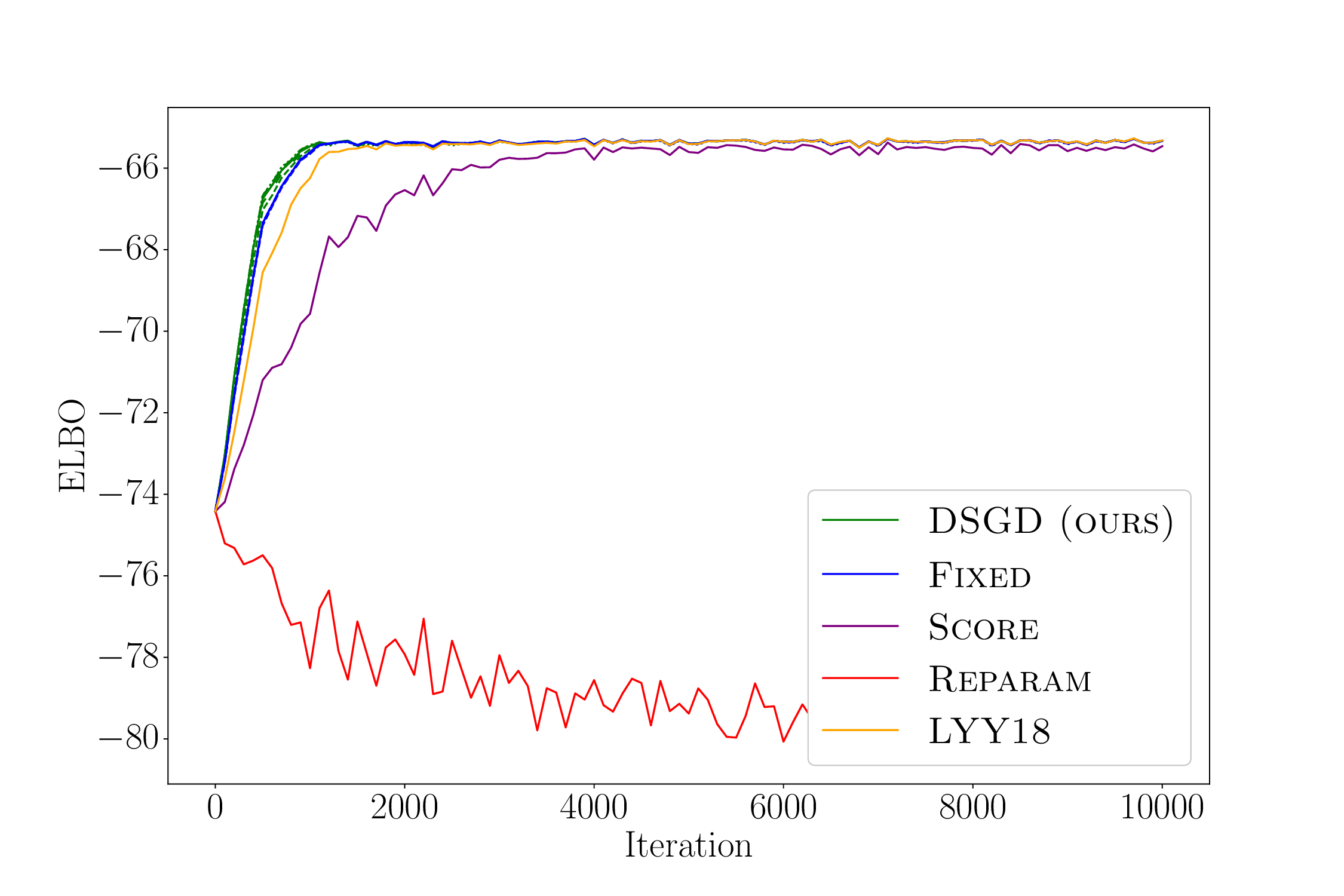}

  \vspace*{-3mm}
  \caption{\texttt{cheating}}
  \label{fig:cheating-graph}
\end{subfigure}

\vspace*{-2mm}
\caption{ELBO trajectories for each model. A single colour is used for each
  estimator and the choice of $\eta=\eta_{4000}=0.06,0.1,0.14,0.18,0.22$ (which determines $\eta_0$) is represented by dashed, loosely dashed, solid,  dash-dotted, dotted lines, respectively.
  \vspace*{-1pt}}
  \label{fig:elbo}
\end{figure}
\paragraph{Experimental Set-Up.}
The (Python) implementation is based on \citep{KOW23,LYY18}. We employ the \texttt{jax} library to provide automatic differentiation which is used to implement each of the above estimators for an arbitrary (probabilistic) program.
The smoothed interpretation can be obtained automatically by (recursively) replacing conditionals $\ifc {E_1}{E_2}{E_3}$ with $\smooth(-E_1)\cdot E_2+\smooth(E_1)\cdot E_3$ in a preprocessing step. (We avoid a potential blowup by using an auxiliary variable for $E_1$.)

In view of \cref{thm:corr}, for \textsc{DSGD} we choose the accuracy coefficient schemes $\eta_k\defeq\eta_0/\sqrt k$ for $\eta_0>0$; due to the nesting of guards we use $\eta_k\defeq\eta_0\cdot k^{-0.2}$ for \texttt{xornet}.
We compare (using the same line style) \textsc{DSGD} for different choices\footnote{The choice of benchmarked hyperparameters for \cref{fig:elbo} expands the range by \cite{KOW23}, who use $0.1,$ $0.15,0.2$, and we uniformly split the range in steps of $0.04$.} of $\eta_0$ to \textsc{Fixed} using the fixed accuracy coefficient corresponding to $\eta_{4000}$.

To enable a fair comparison to \citep{KOW23,LYY18}, we follow their set-up and use the state-of-the-art stochastic optimiser \texttt{Adam}\footnote{together with the respective gradient estimators, e.g.\ $\nabla\sem F_{\eta_k}(\rep_\para(\lat))$ in step $k$ for \textsc{DSGD}} with a step size of \texttt{0.001}, except for \texttt{xornet} for which we use \texttt{0.01}, for 10,000 iterations.
For each iteration, we use $16$ Monte Carlo samples from the chosen estimator to compute the gradient.
As in \citep{LYY18}, the \textsc{LYY18} estimator does not compute the boundary surface term exactly, but estimates it using a single subsample.

% \dw{Do you think it's necessary to give more details how the smoothing is computed, i.e.\ something along the lines of the paragraph ``Compilation for First-Order Programs'' in ESOP?}
% \lo{I'd add this to the appendix.}
 
% Finally, our \texttt{jax} implementation of the \textsc{LYY18} estimator followed the inventors' implementation as closely as possible so we claim it
% is a fair comparison point.
% To support our claim, in the supplementary material we include results from adding our \textsc{Smooth} estimator to the code made publicly available by \citet{LYY18} and show that the results are still comparable.

For every $100$ iterations, we take $1000$ samples of the estimator to estimate the current ELBO value and the variance of the gradient.
Since the gradient is a vector, the variance is taken in two ways: averaging the component-wise variances and the variance of the L2 norm.

We separately benchmark each estimator by computing the number of iterations each can complete in a fixed time budget;
the computational cost of each estimator is then estimated to be the reciprocal of this number.
This then allows us to compute a set of \emph{work-normalised} variances \citep{Botev2017} for each estimator, which are the product of the computational cost and the variance\footnote{This is a more suitable measure than ``raw'' variances since the latter can be improved routinely at the expense of computational efficiency by taking more samples.}.

% \subsection{Models}
% \label{sec:mod}

\begin{table*}[t]
  \caption{Mean of the final ELBO (the higher the better)
  for different random seeds and indicating error bars (the $\pm$ is one standard deviation).}
  \label{tab:finalelbo}
  \vspace*{2pt}
  \begin{subtable}{\linewidth}
    \caption{\texttt{temperature}}
    \centering
  \begin{tabular}{c|ccccc}
    \toprule
    $\eta/\eta_{4000}$ & \textsc{DSGD} (ours) & \textsc{Fixed} & \textsc{Score} & \textsc{Reparam} & \textsc{LYY18}\\
    \midrule
    0.06 & -76 $\pm$ 1 & -624,250 $\pm$ 44,121\\
    0.1 & -84 $\pm$ 2 & -425 $\pm$ 9 & -2,611,479 $\pm$ 255,193 & -706,729 $\pm$ 4,697 & -17,502 $\pm$ 52,044\\
    0.14 & -15,476 $\pm$ 4,641 & -121,932 $\pm$ 85,460\\
    \bottomrule
  \end{tabular}
\end{subtable}

\vspace*{5pt}
\begin{subtable}{\linewidth}
  \caption{\texttt{xornet}}
  \centering
\begin{tabular}{c|ccccc}
  \toprule
  $\eta/\eta_{4000}$ & \textsc{DSGD} (ours) & \textsc{Fixed} & \textsc{Score} & \textsc{Reparam}\\
  \midrule
  0.06 & -3,530 $\pm$ 3,889 & -5,522 $\pm$ 4,136\\
  0.1 & -33 $\pm$ 7 & -2,029 $\pm$ 3,305 & -553 $\pm$ 1,507 & -9,984 $\pm$ 38
\\
  0.14 & -27 $\pm$ 4 & -2,028 $\pm$ 3,986\\
  \bottomrule
\end{tabular}
\end{subtable}
\end{table*}

\begin{table}
  \caption{Computational cost and work-normalised variances, all given as ratios with respect to the \textsc{Score} estimator (omitted since it would be all $1$s).}
  \label{tab:var}

  \vspace*{2pt}
  \begin{subtable}{\linewidth}
    \caption{\texttt{temperature}}
    \label{tab:temperature}
    \centering
    \begin{tabular}{lcccr}
      \toprule
      Estimator & Cost & $\textrm{Avg}(V(.))$ & $V(\|.\|_2)$ \\
      \midrule
      \textsc{DSGD} (ours) &  1.71 &   4.91e-11 &   2.54e-10\\
      \textsc{Fixed} &   1.71 &  2.84e-10 &   2.24e-09 \\
      \textsc{Reparam} &  1.26 &  1.47e-08  &  1.94e-08\\
      \textsc{LYY18} &   9.61 & 1.05e-06   & 4.04e-05 \\
      \bottomrule
    \end{tabular}
  \end{subtable}

  \vspace*{5pt}
  \begin{subtable}{\linewidth}
    \caption{\texttt{xornet}}
    % \label{tab:s}
    \centering
    \begin{tabular}{lcccr}
      \toprule
      Estimator & Cost & $\textrm{Avg}(V(.))$ & $V(\|.\|_2)$ \\
      \midrule
      \textsc{DSGD} (ours) &   1.74 & 6.21e-03  &  3.66e-02 \\
      \textsc{Fixed} &    1.87 &   1.21e-02  &  5.43e-02  \\
      \textsc{Reparam} &   0.388  &  8.34e-09  &  2.62e-09 \\
      \bottomrule
    \end{tabular}
  \end{subtable}

  \vspace*{5pt}
  \begin{subtable}{\linewidth}
    \caption{\texttt{random-walk}}
    \label{tab:w}
    \centering
    \begin{tabular}{lcccr}
      \toprule
      Estimator & Cost & $\textrm{Avg}(V(.))$ & $V(\|.\|_2)$ \\
      \midrule
      \textsc{DSGD} (ours) &  4.70  &  1.71e-01   & 2.61e-01\\
      \textsc{Fixed} &   4.70 &  9.50e-01  &  1.49\\
      \textsc{Reparam} & 2.17 &  8.63e-10   & 7.01e-10 \\
      \textsc{LYY18} & 4.81 &  7.92   & 1.26e+01\\
      \bottomrule
    \end{tabular}
  \end{subtable}

  \vspace*{5pt}
  \begin{subtable}{\linewidth}
    \caption{\texttt{cheating}}
    \label{tab:cheating}
    \centering
    \begin{tabular}{lcccr}
      \toprule
      Estimator & Cost & $\textrm{Avg}(V(.))$ & $V(\|.\|_2)$ \\
      \midrule
      \textsc{DSGD} (ours)&   1.52 &   2.31e-03  &  3.51e-03 \\
      \textsc{Fixed} &   1.52 &  2.84e-03 &   4.64e-03 \\
      \textsc{Reparam} & 9.36e-01 &   4.14e-19  &  1.16e-18\\
      \textsc{LYY18} &   2.59 &   4.27e-02  &  1.09e-01\\
      \bottomrule
    \end{tabular}
  \end{subtable}
\end{table}

\paragraph{Analysis of Results.}
The ELBO trajectories as well as the data for computational cost and variance are presented \cref{fig:elbo,tab:var} (additional models are covered in \cref{app:eval}).
\changed[dw]{Besides, \cref{tab:finalelbo} lists the mean and standard deviation of the final ELBO across different seeds for the random number generator.}
Empirically, the bias of \textsc{Reparam} becomes evident and \textsc{Score} exhibits very high variance, resulting in slow convergence or even inferior results (\cref{fig:temperature-graph}).

Whenever the \textsc{LYY18} estimator is applicable\footnote{For \texttt{xornet}, \textsc{LYY18} is \emph{not} applicable as there are non-affine conditions in if-statements.}, the trajectories for \textsc{DSGD} perform comparably, however \textsc{DSGD} attains orders of magnitude reduction in work-normalised variance (4 to 20,000 x).

Compared to \textsc{Fixed}, we observe that \textsc{DSGD} is more robust\footnote{On a \emph{finite} run our \emph{asymptotic} convergence result \cref{thm:corr} cannot completely eliminate the dependence on this choice.} to the choice of (initial) accuracy coefficients (especially for \texttt{temperature} and \texttt{xornet}). Besides, there is a moderate improvement of variance (\cref{tab:var}).
%  \dw{This suggests that \cref{rem:avvar} is conservative.} 
% \dw{Should we comment that one reason why the choice still matters is that the correctness result is asymptotic but the time budget is fixed in advance so only certain accuracy coefficients can possibly be explored? Or does this downplay the contribution?}

%
%
% This is partly because our estimator has a comparable computational cost to the \textsc{Reparam} estimator,
% but also because it has a low variance, likely from differentiability of the smoothed models.

% Experiments with our prototype implementation confirm the benefits of reduced variance and unbiasedness.
% Compared to the unbiased correction of the reparameterised gradient estimator,
% our estimator has a similar convergence, but is simpler, faster, and attains orders of magnitude (2 to 3,000 x) reduction in work-normalised variance.
% \dw{also better than Fixed}
% }

% \bk{Besides the experimental data, the implementations themselves of the estimators
% show that constructing the \textsc{LYY18} estimator takes a considerable amount
% more of code and restrictions on the model compared to the \textsc{Smooth}
% estimator, which is quite similar to the \textsc{Reparam} estimator in terms of
% code complexity.
% }

\section{CONCLUDING REMARKS}
\label{sec:conc}

We have proposed a variant of SGD, \emph{Diagonalisation Stochastic Gradient Descent}, and shown \emph{provable} correctness. Our approach is based on a smoothed interpretation of (possibly) discontinuous programs, which also yields unbiased gradient estimators.
Crucially, asymptotically a stationary point of the original, \emph{unsmoothed} problem is attained and a hyperparameter (accuracy of approximations) is tuned automatically. The correctness hinges on a careful analysis of the variance and a compatible scheme governing the accuracies. Notably, this purely depends on the (syntactic) structure of the program.

% We have designed a type system to impose a mild restriction on the programming language.
% Whilst this provides sufficient conditions for theoretical guarantees, we stress that DSGD and the smoothed unbiased gradient estimator can even be applied to programs which are \emph{not} typable.

Our experimental evaluation demonstrates important advantages over the state of the art: significantly lower variance (score estimator), unbiasedness (reparametrisation estimator), simplicity, wider applicability and lower variance (unbiased correction thereof), as well as stability over the choice of (initial) accuracy coefficients (fixed smoothing).

% Experiments with our prototype implementation confirm the benefits of reduced variance and unbiasedness.

% Compared to the unbiased correction of the reparameterised gradient estimator,
% our estimator has a similar convergence, but is simpler, faster, and attains orders of magnitude (2 to 3,000 x) reduction in work-normalised variance.
% \dw{also better than Fixed}

\vspace*{-2mm}
\paragraph{Limitations and Future Directions.}
Our ana\-lysis is asymptotic and focuses on stationary points, which leaves room for future research (convergence rates, avoidance of saddle points etc.).

\iffalse
We aim to extend our programming language whilst still retaining the pleasing properties derived here. In particular, we would like to support certain \emph{continuous but non-differentiable} primitives such as ReLU \emph{without smoothing} them. Besides, we would like to support recursion to express non-parametric models.
\fi

Furthermore, we plan to explore methods \emph{adaptively} tuning the accuracy coefficient rather than \emph{a priori} fixing a scheme. Whilst the present work was primarily concerned with theoretical guarantees, we anticipate adaptive methods to outperform fixed schemes in practice (but likely without theoretical guarantees).

\changed[dw]{
\subsection*{Acknowledgements}
This research is supported by the National Research Foundation, Singapore, under its RSS Scheme (NRF-RSS2022-009).
}

\bibliographystyle{apalike}
\bibliography{lit.bib}

\begin{thebibliography}{}

\bibitem[Barthe et~al., 2020]{BKS20}
Barthe, G., Katoen, J.-P., and Silva, A., editors (2020).
\newblock {\em Foundations of Probabilistic Programming}.
\newblock Cambridge University Press.

\bibitem[Bertsekas, 2015]{B15}
Bertsekas, D. (2015).
\newblock {\em Convex optimization algorithms}.
\newblock Athena Scientific.

\bibitem[Bertsekas and Tsitsiklis, 2000]{BT00}
Bertsekas, D.~P. and Tsitsiklis, J.~N. (2000).
\newblock Gradient convergence in gradient methods with errors.
\newblock {\em {SIAM} J. Optim.}, 10(3):627--642.

\bibitem[Bingham et~al., 2019]{pyro}
Bingham, E., Chen, J.~P., Jankowiak, M., Obermeyer, F., Pradhan, N.,
  Karaletsos, T., Singh, R., Szerlip, P.~A., Horsfall, P., and Goodman, N.~D.
  (2019).
\newblock Pyro: Deep universal probabilistic programming.
\newblock {\em J. Mach. Learn. Res.}, 20:28:1--28:6.

\bibitem[Blake and Zisserman, 1987]{BZ87}
Blake, A. and Zisserman, A. (1987).
\newblock {\em Visual Reconstruction}.
\newblock {MIT} Press.

\bibitem[Blei et~al., 2017]{BKM17}
Blei, D.~M., Kucukelbir, A., and McAuliffe, J.~D. (2017).
\newblock Variational inference: A review for statisticians.
\newblock {\em Journal of the American Statistical Association},
  112(518):859--877.

\bibitem[Botev and Ridder, 2017]{Botev2017}
Botev, Z. and Ridder, A. (2017).
\newblock {Variance Reduction}.
\newblock In {\em Wiley StatsRef: Statistics Reference Online}, pages 1--6.

\bibitem[Davidson-Pilon, 2015]{DavidsonPilon15}
Davidson-Pilon, C. (2015).
\newblock {\em Bayesian Methods for Hackers: Probabilistic Programming and
  Bayesian Inference}.
\newblock Addison-Wesley Professional.

\bibitem[Fu, 2006]{F06}
Fu, M.~C. (2006).
\newblock Chapter 19 gradient estimation.
\newblock In Henderson, S.~G. and Nelson, B.~L., editors, {\em Simulation},
  volume~13 of {\em Handbooks in Operations Research and Management Science},
  pages 575--616. North-Holland.

\bibitem[Hazan et~al., 2016]{HLS16}
Hazan, E., Levy, K.~Y., and Shalev{-}Shwartz, S. (2016).
\newblock On graduated optimization for stochastic non-convex problems.
\newblock In Balcan, M. and Weinberger, K.~Q., editors, {\em Proceedings of the
  33nd International Conference on Machine Learning, {ICML} 2016, New York
  City, NY, USA, June 19-24, 2016}, volume~48 of {\em {JMLR} Workshop and
  Conference Proceedings}, pages 1833--1841. JMLR.org.

\bibitem[H{\"o}rmander, 2015]{H15}
H{\"o}rmander, L. (2015).
\newblock {\em The Analysis of Linear Partial Differential Operators I:
  Distribution Theory and Fourier Analysis}.
\newblock Classics in Mathematics. Springer Berlin Heidelberg.

\bibitem[Jang et~al., 2017]{JGP17}
Jang, E., Gu, S., and Poole, B. (2017).
\newblock Categorical reparameterization with gumbel-softmax.
\newblock In {\em 5th International Conference on Learning Representations,
  {ICLR} 2017, Toulon, France, April 24-26, 2017, Conference Track
  Proceedings}.

\bibitem[Khajwal et~al., 2023]{KOW23}
Khajwal, B., Ong, C.~L., and Wagner, D. (2023).
\newblock Fast and correct gradient-based optimisation for probabilistic
  programming via smoothing.
\newblock In Wies, T., editor, {\em Programming Languages and Systems - 32nd
  European Symposium on Programming, {ESOP} 2023, Held as Part of the European
  Joint Conferences on Theory and Practice of Software, {ETAPS} 2023, Paris,
  France, April 22-27, 2023, Proceedings}, volume 13990 of {\em Lecture Notes
  in Computer Science}, pages 479--506. Springer.

\bibitem[Klenke, 2014]{K13}
Klenke, A. (2014).
\newblock {\em Probability Theory: A Comprehensive Course}.
\newblock Universitext. Springer London.

\bibitem[Lee et~al., 2018]{LYY18}
Lee, W., Yu, H., and Yang, H. (2018).
\newblock Reparameterization gradient for non-differentiable models.
\newblock In {\em Advances in Neural Information Processing Systems 31: Annual
  Conference on Neural Information Processing Systems 2018, NeurIPS 2018, 3-8
  December 2018, Montr{\'{e}}al, Canada}, pages 5558--5568.

\bibitem[Maddison et~al., 2017]{MMT17}
Maddison, C.~J., Mnih, A., and Teh, Y.~W. (2017).
\newblock The concrete distribution: {A} continuous relaxation of discrete
  random variables.
\newblock In {\em 5th International Conference on Learning Representations,
  {ICLR} 2017, Toulon, France, April 24-26, 2017, Conference Track
  Proceedings}.

\bibitem[Mak et~al., 2021]{MOPW}
Mak, C., Ong, C.~L., Paquet, H., and Wagner, D. (2021).
\newblock Densities of almost surely terminating probabilistic programs are
  differentiable almost everywhere.
\newblock In Yoshida, N., editor, {\em Programming Languages and Systems - 30th
  European Symposium on Programming, {ESOP} 2021, Held as Part of the European
  Joint Conferences on Theory and Practice of Software, {ETAPS} 2021,
  Luxembourg City, Luxembourg, March 27 - April 1, 2021, Proceedings}, volume
  12648 of {\em Lecture Notes in Computer Science}, pages 432--461. Springer.

\bibitem[Minh and Gregor, 2014]{DBLP:conf/icml/MnihG14}
Minh, A. and Gregor, K. (2014).
\newblock Neural variational inference and learning in belief networks.
\newblock In {\em Proceedings of the 31th International Conference on Machine
  Learning, {ICML} 2014, Beijing, China, 21-26 June 2014}, volume~32 of {\em
  {JMLR} Workshop and Conference Proceedings}, pages 1791--1799. JMLR.org.

\bibitem[Mityagin, 2015]{M15}
Mityagin, B. (2015).
\newblock The zero set of a real analytic function.

\bibitem[Mohamed et~al., 2020]{MRFM20}
Mohamed, S., Rosca, M., Figurnov, M., and Mnih, A. (2020).
\newblock Monte carlo gradient estimation in machine learning.
\newblock {\em J. Mach. Learn. Res.}, 21:132:1--132:62.

\bibitem[Ranganath et~al., 2014]{RGB14}
Ranganath, R., Gerrish, S., and Blei, D.~M. (2014).
\newblock Black box variational inference.
\newblock In {\em Proceedings of the Seventeenth International Conference on
  Artificial Intelligence and Statistics, {AISTATS} 2014, Reykjavik, Iceland,
  April 22-25, 2014}, pages 814--822.

\bibitem[Reed and Simon, 2003]{R03}
Reed, M. and Simon, B. (2003).
\newblock {\em Methods of Modern Mathematical Physics: Functional analysis. I}.
\newblock World Published Corporation.

\bibitem[Rezende et~al., 2014]{DBLP:conf/icml/RezendeMW14}
Rezende, D.~J., Mohamed, S., and Wierstra, D. (2014).
\newblock Stochastic backpropagation and approximate inference in deep
  generative models.
\newblock In {\em Proceedings of the 31th International Conference on Machine
  Learning, {ICML} 2014, Beijing, China, 21-26 June 2014}, volume~32 of {\em
  {JMLR} Workshop and Conference Proceedings}, pages 1278--1286. JMLR.org.

\bibitem[Robbins and Monro, 1951]{RM51}
Robbins, H. and Monro, S. (1951).
\newblock A stochastic approximation method.
\newblock {\em The annals of mathematical statistics}, pages 400--407.

\bibitem[Ruiz et~al., 2016]{RTB16b}
Ruiz, F. J.~R., Titsias, M.~K., and Blei, D.~M. (2016).
\newblock The generalized reparameterization gradient.
\newblock In {\em Advances in Neural Information Processing Systems 29: Annual
  Conference on Neural Information Processing Systems 2016, December 5-10,
  2016, Barcelona, Spain}, pages 460--468.

\bibitem[Schulman et~al., 2015]{SHWA15}
Schulman, J., Heess, N., Weber, T., and Abbeel, P. (2015).
\newblock Gradient estimation using stochastic computation graphs.
\newblock In Cortes, C., Lawrence, N.~D., Lee, D.~D., Sugiyama, M., and
  Garnett, R., editors, {\em Advances in Neural Information Processing Systems
  28: Annual Conference on Neural Information Processing Systems 2015, December
  7-12, 2015, Montreal, Quebec, Canada}, pages 3528--3536.

\bibitem[Shumway and Stoffer, 2005]{ShumwayS05}
Shumway, R.~H. and Stoffer, D.~S. (2005).
\newblock {\em Time Series Analysis and Its Applications}.
\newblock Springer Texts in Statistics. Springer-Verlag.

\bibitem[Soudjani et~al., 2017]{DBLP:conf/qest/SoudjaniMN17}
Soudjani, S. E.~Z., Majumdar, R., and Nagapetyan, T. (2017).
\newblock Multilevel monte carlo method for statistical model checking of
  hybrid systems.
\newblock In Bertrand, N. and Bortolussi, L., editors, {\em Quantitative
  Evaluation of Systems - 14th International Conference, {QEST} 2017, Berlin,
  Germany, September 5-7, 2017, Proceedings}, volume 10503 of {\em Lecture
  Notes in Computer Science}, pages 351--367. Springer.

\bibitem[Sutton and Barto, 2018]{SB98}
Sutton, R.~S. and Barto, A.~G. (2018).
\newblock {\em Reinforcement Learning: An Introduction}.
\newblock The MIT Press, second edition.

\bibitem[Tucker et~al., 2017]{TMMLS17}
Tucker, G., Mnih, A., Maddison, C.~J., Lawson, D., and Sohl{-}Dickstein, J.
  (2017).
\newblock {REBAR:} low-variance, unbiased gradient estimates for discrete
  latent variable models.
\newblock In Guyon, I., von Luxburg, U., Bengio, S., Wallach, H.~M., Fergus,
  R., Vishwanathan, S. V.~N., and Garnett, R., editors, {\em Advances in Neural
  Information Processing Systems 30: Annual Conference on Neural Information
  Processing Systems 2017, December 4-9, 2017, Long Beach, CA, {USA}}, pages
  2627--2636.

\bibitem[van~de Meent et~al., 2018]{MPYW18}
van~de Meent, J., Paige, B., Yang, H., and Wood, F. (2018).
\newblock An introduction to probabilistic programming.
\newblock {\em CoRR}, abs/1809.10756.

\bibitem[Wingate and Weber, 2013]{WW13}
Wingate, D. and Weber, T. (2013).
\newblock Automated variational inference in probabilistic programming.
\newblock {\em CoRR}, abs/1301.1299.

\bibitem[Xu et~al., 2019]{XQKS19}
Xu, M., Quiroz, M., Kohn, R., and Sisson, S.~A. (2019).
\newblock Variance reduction properties of the reparameterization trick.
\newblock In Chaudhuri, K. and Sugiyama, M., editors, {\em The 22nd
  International Conference on Artificial Intelligence and Statistics, {AISTATS}
  2019, 16-18 April 2019, Naha, Okinawa, Japan}, volume~89 of {\em Proceedings
  of Machine Learning Research}, pages 2711--2720. {PMLR}.

\bibitem[Zang, 1981]{Z81}
Zang, I. (1981).
\newblock Discontinuous optimization by smoothing.
\newblock {\em Mathematics of Operations Research}, 6(1):140--152.

\bibitem[Zhang et~al., 2019]{ZBKM19}
Zhang, C., B{\"{u}}tepage, J., Kjellstr{\"{o}}m, H., and Mandt, S. (2019).
\newblock Advances in variational inference.
\newblock {\em {IEEE} Trans. Pattern Anal. Mach. Intell.}, 41(8):2008--2026.

\end{thebibliography}

\section*{Checklist}

 \begin{enumerate}

 \item For all models and algorithms presented, check if you include:
 \begin{enumerate}
   \item A clear description of the mathematical setting, assumptions, algorithm, and/or model. [\textbf{Yes}/No/Not Applicable]
   
   \textit{We present the setting in \cref{sec:setup} and state \cref{ass:basic,ass:comp}. The DSGD algorithm is presented in \cref{sec:dsgd}.}
   \item An analysis of the properties and complexity (time, space, sample size) of any algorithm. [\textbf{Yes}/No/Not Applicable]
   
   \textit{We prove an (asymptotic) convergence result \cref{prop:cdsgd,thm:corr}.}
   \item (Optional) Anonymized source code, with specification of all dependencies, including external libraries. [\textbf{Yes}/No/Not Applicable]
   
   \textit{The Python code is at \url{https://github.com/domwagner/DSGD.git}. The dependencies are} \texttt{jax},
   \texttt{jaxlib},
   \texttt{numpy},
   \texttt{ipykernel},
   \texttt{matplotlib}.
 \end{enumerate}

 \item For any theoretical claim, check if you include:
 \begin{enumerate}
   \item Statements of the full set of assumptions of all theoretical results. [\textbf{Yes}/No/Not Applicable]
   
   \textit{See \cref{ass:basic,ass:comp}.}
   \item Complete proofs of all theoretical results. [\textbf{Yes}/No/Not Applicable]
   
   \textit{Proofs are presented in \cref{app:smooth,app:dsgd,app:precon}.}
   \item Clear explanations of any assumptions. [\textbf{Yes}/No/Not Applicable]     
   
   \textit{We discuss \cref{ass:basic,eq:strdiff} in \cref{sec:setup,sec:unicon}, respectively.}
 \end{enumerate}

 \item For all figures and tables that present empirical results, check if you include:
 \begin{enumerate}
   \item The code, data, and instructions needed to reproduce the main experimental results (either in the supplemental material or as a URL). [\textbf{Yes}/No/Not Applicable]
   
   \textit{The code is available at \url{https://github.com/domwagner/DSGD.git}.
   The experiements can be viewed and run in the jupyter notebook} \texttt{experiments.ipynb} \textit{by running:}

\texttt{jupyter notebook experiments.ipynb}.

   \item All the training details (e.g., data splits, hyperparameters, how they were chosen). [\textbf{Yes}/No/Not Applicable]
   
   \textit{We present the details \cref{sec:eval,fig:elbo} and closely follow the setup of \cite{LYY18,KOW23}. In particular, the choice of the benchmarked hyperparameters $\eta=\eta_{4000}=0.06,0.1,0.14,0.18,0.22$ for \cref{fig:elbo} expands the range by \cite{KOW23}, who use $0.1,0.15,0.2$, and we uniformly split the range in steps of $0.04$.}
         \item A clear definition of the specific measure or statistics and error bars (e.g., with respect to the random seed after running experiments multiple times). [\textbf{Yes}/No/Not Applicable]
         
         \textit{We describe our approach to compute the work-normalised variance and the computational cost in \cref{sec:eval}.}
         \item A description of the computing infrastructure used. (e.g., type of GPUs, internal cluster, or cloud provider). [\textbf{Yes}/No/Not Applicable]
         
         \textit{We run our experiments on a MacBook Air (13-inch, 2017) with Intel HD Graphics 6000 1536 MB.}
 \end{enumerate}

 \item If you are using existing assets (e.g., code, data, models) or curating/releasing new assets, check if you include:
 \begin{enumerate}
   \item Citations of the creator If your work uses existing assets. [Yes/No/\textbf{Not Applicable}]
   \item The license information of the assets, if applicable. [Yes/No/\textbf{Not Applicable}]
   \item New assets either in the supplemental material or as a URL, if applicable. [Yes/No/\textbf{Not Applicable}]
   \item Information about consent from data providers/curators. [Yes/No/\textbf{Not Applicable}]
   \item Discussion of sensible content if applicable, e.g., personally identifiable information or offensive content. [Yes/No/\textbf{Not Applicable}]
 \end{enumerate}

 \item If you used crowdsourcing or conducted research with human subjects, check if you include:
 \begin{enumerate}
   \item The full text of instructions given to participants and screenshots. [Yes/No/\textbf{Not Applicable}]
   \item Descriptions of potential participant risks, with links to Institutional Review Board (IRB) approvals if applicable. [Yes/No/\textbf{Not Applicable}]
   \item The estimated hourly wage paid to participants and the total amount spent on participant compensation. [Yes/No/\textbf{Not Applicable}]
 \end{enumerate}

 \end{enumerate}

 \appendix
 \newpage
 \onecolumn
% !TEX root = ../main.tex

\section{Supplementary Materials for \cref{sec:setup}}
\label{app:setup}
% \begin{table}

  $F\in\sfct$ naturally defines a function $\Real^n\to\Real$, which we denote by $\sem F$:
\begin{align*}
  \sem{z_j}(\repx)&\defeq z_j\\
  \sem{f(F_1,\ldots,F_k)}(\repx)&\defeq f(\sem{F_1}(\repx),\ldots,\sem{F_k}(\repx))\\
  \sem{\ifc FGH}(\repx)&\defeq
  \begin{cases}
    \sem G(\repx)&\tif\sem F(\repx)<0\\
    \sem H(\repx)&\tow
  \end{cases}
\end{align*}

% \subsection{Supplementary Materials for \cref{sec:schw}}

\begin{table}[h]
  \begin{center}
    \caption{Schwartz Distributions and their PDFs}
    \vspace*{8pt}

  \begin{tabular}{ c|ccc }
   distribution & pdf  & support & parameters\\ \hline
   normal & $\frac 1 {\sigma\sqrt{2\pi}}\exp\left(\frac {-(x-\mu)^2}{2\sigma^2}\right)$ & $\Real$ & $\mu\in\Real, \sigma\in\Realp$\\[3pt]
   half normal & $\frac {\sqrt 2} {\sigma\sqrt{\pi}}\exp\left(\frac {-x^2}{2\sigma^2}\right)$ & $\Realnn$ &  $\sigma\in\Realp$\\[3pt]
   exponential & $\lambda\cdot\exp(-\lambda x)$ & $\Realnn$ & $\lambda\in\Realp$ \\[3pt]
   % gamma & $\frac 1{\Gamma(k)\theta^k}x^{k-1}\exp\left(-\frac x \theta\right)$&$\Realp$& $k,\theta\in\Realp$\\[3pt]
   % Weibull & $\frac k\lambda \left(\frac x\lambda\right)^{k-1}\exp\left(-\frac x\lambda\right)$&$\Realnn$&$\lambda,k\in\Realp$\\[3pt]
   logistic & $\frac {\exp\left(-\frac {x-\mu}{s}\right)}{s\left(1+\exp\left(-\frac {x-\mu}{s}\right)\right)^2}$&$\Real$&$\mu\in\Real, s\in\Realp$\\
   % gamma & $\frac 1{\Gamma(k)\theta^k}\cdot x^{k-1}\cdot\exp(-\frac x\theta)$& $\Real_{>0}$ & $k,\theta\in\Real_{>0}$
  \end{tabular}
  \end{center}
  \label{tab:dens}
\end{table}

\section{Supplementary Materials for \cref{sec:smooth}}
\label{app:smooth}

The following immediately follows from a well-known result about exchanging differentiation and integration, which is a consequence of the dominated convergence theorem \citep[Theorem 6.28]{K13}:
\begin{lemma}
  \label{lem:intdiff}
  Let $\Theta \subseteq\Real$ be open and $\suppd\subseteq\Real^n$ be measurable. If $f\from\Theta\times\suppd\to\Real$ satisfies
  \begin{enumerate}
    \item for each $\theta\in\Theta$, $\lat\mapsto f(\theta,\lat)$ is integrable
    \item $f$ is differentiable %for almost all $\lat\in\Real^n$, for all $\theta\in\Theta$, $\frac {\partial f}{\partial\theta}$ exists
    \item there exists an integrable $g\from \suppd\to\Real$ satisfying $|\frac{\partial f}{\partial \theta}(\theta,\lat)|\leq g(\lat)$ for all $(\theta,\lat)\in\Theta\times \suppd$.
  \end{enumerate}
  then for all $\theta\in\Theta$, $\frac\partial{\partial\theta}\int_\suppd f(\theta,\lat)\diff\lat=\int_\suppd\frac{\partial f}{\partial\theta}(\theta,\lat)\diff\lat$.
\end{lemma}
Note that the second premise fails for the function $f$ in \cref{ex:biased}: for \emph{all} $s\in\Real$, $\frac{\partial f}{\partial\theta}(-s,s)$ does not exist.

\unb*
\begin{proof}
  Let $\para\in\parasp$.
  We apply \cref{lem:intdiff} to some ball $\parasp'\subseteq\parasp$ around $\para$ and $(\lat,\para)\mapsto\mdist(\lat)\cdot\sema F(\rep_\para(\lat))$. We have already seen the well-definedness (first premise) and the second is obvious (since $\sema F$ is a smoothing).
  For the third premise, we observe that for each $\eta>0$, $\frac{\partial(\sema F\circ\rep_{(-)})}{\partial\theta_i}$ is bounded by a polynomial (using \cref{ass:basic}). Therefore, by \cref{lem:polyp} (below) there exists a polynomial $p:\Real^n\to\Real$ satisfying $\frac{\partial(\sema F\circ\rep_{(-)})}{\partial\theta_i}(\para,\lat)\leq |p(\lat)|$ for all $(\para,\lat)\in\parasp'\times\Real^n$ and integrability of $\mdist(\lat)\cdot p(\lat)$ follows with \cref{lem:Schwprop}.
\end{proof}

\begin{lemma}
  \label{lem:polyp}
  \begin{enumerate}
    \item\label{it:polymon} If $p:\Real^n\to\Real$ is a polynomial then there exists a polynomial $q$ such that for all $|x_1|\leq |x'_1|,\ldots, |x_n|\leq |x_n'|$,
    $|p(\mathbf x)|\leq q(\mathbf x')$.
    \item\label{lem:polyunib} If $f:\parasp\times\Real^n\to\Real$ is bounded by a polynomial, where $\parasp\subseteq\Real^m$ is compact then there exists a polynomial $p:\Real^n\to\Real$ satisfying $|f(\para,\mathbf x)|\leq p(\mathbf x)$ for all $(\para,\mathbf x)\in\parasp\times\Real^n$. % such that for all $(\para,\repx)\in \parasp\times\Real^n$, $|f(\para,\repx)|\leq p(\repx)$.
  \end{enumerate}
\end{lemma}
For example for $p(x_1,x_2)= x_1^2\cdot x_2 - x_2$ the polynomial $x_1^2\cdot(x_2^2+1)+(x_2^2+1)$ satisfies this property. (The following proof yields $((x_1^2+1)^2+1)\cdot (x_2^2+1)+(x_2^2+1)$.) 
Besides, $p(\theta,x_1,x_2)= \theta\cdot x_1\cdot x_2 - x_2$ is uniformly bounded by $2\cdot (x_1^2+1)\cdot(x_2^2+1)+(x_2^2+1)$ on $\Theta=(-2,1)$.
\begin{proof}
  \begin{enumerate}
\item If $p(\mathbf x)=x_i$ then we can choose $q(\mathbf x)\defeq(x_i^2+1)$ because for $|x_i|\leq |x'_i|$, $|x_i|\leq |x_i'|<(x_i')^2+1$.

If $p(\mathbf x)=c$ for $c\in\Real$ then we can choose $q(\mathbf x)\defeq |c|$.

Finally, suppose that $p_1,p_2,q_1,q_2$ are polynomials such that for all $|\mathbf x|\leq|\mathbf x'|$, $|p_1(\mathbf x)|\leq q_1(\mathbf x')$ and $|p_2(\mathbf x)|\leq q_2(\mathbf x')$.
Then for all $|\mathbf x|\leq|\mathbf x'|$,
\begin{align*}
  |p_1(\mathbf x)+p_2(\mathbf x)|&\leq |p_1(\mathbf x)|+|p_2(\mathbf x)|\leq q_1(\mathbf x')+q_2(\mathbf x')\\
  |p_1(\mathbf x)\cdot p_2(\mathbf x)|&\leq |p_1(\mathbf x)|\cdot|p_2(\mathbf x)|\leq q_1(\mathbf x')\cdot q_2(\mathbf x')
\end{align*}
\item If $f:\parasp\times\Real^n\to\Real$ is bounded by a polynomial then by the first part there exists $q:\parasp\times\Real^n\to\Real$ such that for all $(\para,\repx),(\para',\repx)\in \parasp\times\Real^n$ with $|\para|\leq|\para'|$, $|f(\para,\repx)|\leq q(\para,\repx')$. Let $\theta_i^*\defeq\sup_{\para\in\parasp}\theta_i<\infty$ ($\parasp$ is bounded) and $p(\repx)\defeq q(\repx,\para^*)$. Finally, it suffices to note that for every $(\para,\repx)\in\parasp\times\Real^n$,
 $|f(\para,\repx)|\leq q(\para^*,\repx)=p(\repx)$. \qedhere
  \end{enumerate}
\end{proof}

\section{Supplementary Materials for \cref{sec:dsgd}}
\label{app:dsgd}

\cdsgd*
\begin{proof}
  Let $L$ be the constant for Lipschitz smoothness in \ref{it:lsmooth}.
  % Note that \ref{it:lsmooth} and \ref{it:univcd} imply the $L$-Lipschitz smoothness of $g$. \dw{justification below}
  By Taylor's theorem, $\para_k,\para_{k+1}\in\parasp$ and convexity of $\parasp$,
  % \ref{it:lsmooth}, \ref{it:unb} and \ref{it:unibv},
  \begin{align*}
    \E_{\lat_k\sim\mdist}[g(\para_{k+1})]
    &\leq\E_{\lat_k\sim\mdist}\left[g(\para_k)-\gamma_k\cdot\langle\nabla_\para\, g(\para_k),\nabla_\para\, f_k(\para_k,\lat_k)\rangle+\frac{\gamma_k^2}2\cdot L\cdot\|\nabla_\para\, f_k(\para_k,\lat_k)\|^2\right]\\
    &\leq g(\para_k)-\gamma_k\cdot\langle\nabla_\para\, g(\para_k),\nabla_\para\, g_k(\para_k)\rangle+\frac{\gamma_k^2}2\cdot L\cdot V_k
  \end{align*}
  using \ref{it:unb} and \ref{it:bvar} in the second step. Hence,
\begin{align*}
  \gamma_k\cdot\langle\nabla_\para\, g(\para_k),\nabla_\para\, g_k(\para_k)\rangle\leq g(\para_k)-\E_{\lat_k\sim\mdist}[g(\para_{k+1})]+\frac{\gamma_k^2}2\cdot L\cdot V_k
  \end{align*}
  and thus,
  \begin{align*}
    % \label{eq:sum}
    % \left(\sum_{i=0}^k\gamma_i\right)\cdot\min_{i=1}^k\E[ \langle\nabla_\para\, g(\para_i),\nabla_\para\, g_i(\para_i)\rangle]
    % &\leq
    \sum_{i=0}^{k-1}\gamma_i\cdot\E_{\lat_0,\ldots,\lat_{k-1}\sim\mdist}[\langle\nabla_\para\, g(\para_i),\nabla_\para\, g_i(\para_i)\rangle]
    &\leq g(\para_0)-\E[g(\para_{k})]+L\cdot\sum_{i=0}^{k-1}\frac{\gamma_i^2}2\cdot V_i
  \end{align*}
  By boundedness \ref{it:bounded}, $g(\para_0)-\E[g(\para_{k})]\leq c<\infty$ (for some $c>0$) and therefore due to \cref{eq:stepvar} it follows:
  \begin{align}
    \label{eq:dsgdproof}
    \inf_{i\in\nat}\E[\langle\nabla_\para\, g(\para_i),\nabla_\para\, g_i(\para_i)\rangle]\leq \left(\sum_{i\in\nat}\gamma_i\right)^{-1}\cdot\left(c+L\cdot\sum_{i\in\nat}\frac{\gamma_i^2}2\cdot V_i\right)=0
  \end{align}
  By the same reasoning for all $i_0\in\nat$, $\inf_{i\geq i_0}\E[\langle\nabla_\para\, g(\para_i),\nabla_\para\, g_i(\para_i)\rangle]=0$ and therefore also $\liminf_{i\in\nat}\E[\langle\nabla_\para\, g(\para_i),\nabla_\para\, g_i(\para_i)\rangle]=0$.

  Next, observe that
  \begin{align*}
    \|g(\para_i)\|^2&=\langle\nabla_\para\, g(\para_i),\nabla_\para\, g_i(\para_i)\rangle+\langle\nabla_\para\, g(\para_i),\nabla_\para\, g(\para_i)-\nabla_\para\, g_i(\para_i)\rangle\\
    &\leq\langle\nabla_\para\, g(\para_i),\nabla_\para\, g_i(\para_i)\rangle+\|\nabla_\para\, g(\para_i)\|\cdot\|\nabla_\para\, g(\para_i)-\nabla_\para\, g_i(\para_i)\|%\label{eq:bound}
  \end{align*}
  and the right summand uniformly converges to $0$ because of Lipschitz continuity \ref{it:unibg} and the fact that $\|\nabla_\para\, g(\para_i)-\nabla_\para\, g_i(\para_i)\|$ does so \ref{it:univcd}.
Consequently,
\begin{align*}
  \liminf_{i\in\nat}\E\left[\|g(\para_i)\|^2\right]=0
\end{align*}
  Finally, this can only possibly hold if almost surely $\liminf_{i\in\nat}\|g(\para_i)\|^2=0$.
\end{proof}

\section{Supplementary Materials for \cref{sec:precon}}
\label{app:precon}

The following can be viewed as an extension of \cref{lem:Schwprop}(\ref{it:int}) and is useful for \cref{lem:uniintder}:
\begin{restatable}{lemma}{uniint}
  \label{lem:uniint}
  If $f:\parasp\times\suppd\to\Real$ is continuous and satisfies
  \begin{align*}
    \sup_{(\para,\repx)\in\parasp\times\suppd}\|\repx\|^{n+3}\cdot |f(\para,\repx)|<\infty
  \end{align*} then $
    \int_\suppd\sup_{\para\in\parasp} |f(\para,\repx)|\diff\repx<\infty$.
\end{restatable}
\begin{proof}
  Let $g(\repx)\defeq\sup_{\para\in\parasp}|f(\para,\repx)|$ and $U_k\defeq\suppd\cap[-k,k]^n$. Therefore,
  \begin{align*}
    &\int_\suppd g(\repx)\diff\repx\\
    &= \int_{U_1} g(\repx)\diff\repx + \int_{\suppd\setminus U_1} g(\repx)\diff\repx\\
    &\leq 2^n\cdot\sup_{(\para,\repx)\in\parasp\times U_1} |f(\para,\repx)|+ \left(\sup_{(\para,\repx)\in\parasp\times\suppd}\|\repx\|^{n+3}\cdot f(\para,\repx)\right)\cdot\int_{\suppd\setminus U_1} \frac 1{\|\repx\|^{n+3}}\diff\repx
  \end{align*}
  All the terms are finite because (NB if $\repx\in\suppd\setminus U_{k}$ then $k\leq\|\repx\|$)
  \begin{align*}
    \int_{\suppd\setminus U_1} \frac 1{\|\repx\|^{n+3}}\diff\repx&=\sum_{k=1}^\infty \int_{U_{k+1}\setminus U_k} \frac 1{\|\repx\|^{n+3}}\diff\repx\\
    &\leq \sum_{k=1}^\infty \int_{U_{k+1}} \frac 1{k^{n+3}}\diff\repx\\
    &\leq \sum_{k=1}^\infty \frac {(2k)^{n+1}}{k^{n+3}}\diff\repx\\
    &\leq 4^n\cdot\sum_{k=1}^\infty \frac {1}{k^2}\diff\repx<\infty\qedhere
  \end{align*}
\end{proof}

\begin{lemma}
  \label{lem:finmeas}
  If $f:\suppd\to\Real$, where $\suppd\subseteq\Real^n$, is a Schwartz function, $\rep_{(-)}:\parasp\times\Real^n\to\Real^n$ satisfies \cref{ass:basic,eq:strdiff}, and $p:\Real^n\to\Real$ a polynomial then
  \begin{align*}
    \int_\suppd \sup_{\para\in\parasp}\left|f(\rep_{\para}^{-1}(\repx))\cdot \det\Jac\rep_\para^{-1}(\repx)\cdot p(\repx) \right|\diff\repx<\infty
  \end{align*}
\end{lemma}
  \begin{proof}
    \iffalse
    Note that $\det\Jac\rep_\para(\lat)$ is a polynomial $p_1:\parasp\times\Real^n\to\Real$. \dw{not necessary} Since each $\rep_\para$ is a diffeomorphism and $p$ is continuous, $p$ is either positive or negative. W.l.o.g.\ we assume the former.
      Thus,
      \begin{align*}
        % |\det\Jac\rep_\para^{-1}(\repx)|=\frac 1{|\det\Jac\rep_\para(\rep_\para^{-1}(\repx))|}=\frac 1 {p(\para,\rep_\para^{-1}(\repx))}\\
        |\det\Jac\rep_\para^{-1}(\rep_\para(\lat))|=\frac 1{|\det\Jac\rep_\para(\lat)|}=\frac 1 {p_1(\para,\lat)}
      \end{align*}
      \dw{add assumption}
      \fi
      Since $\rep_{(-)}$ satisfies \cref{eq:strdiff}, there exists $c > 0$ such that $|\det\Jac\rep_\para^{-1}(\repx)| \leq c$ and
      % By \cite[Lemma~2]{H58} there exists $c>0$ such that $p_1(\para,\lat)>c$ and
      by \cref{ass:basic,lem:polyp} there exists a polynomial $q:\Real^n\to\Real$ satisfying 
      \begin{align*}
        \left|\|\rep_\para(\lat)\|^{n+3}+p(\rep_\para(\lat))\right|\leq q(\lat)
      \end{align*}
      for all $(\para,\lat)\in\parasp\times\suppd$.
      Hence,
      \begin{align*}
        &\sup_{(\para,\repx)\in\parasp\times\suppd}\|\repx\|^{n+3}\cdot \left|f(\rep_{\para}^{-1}(\repx))\cdot \det\Jac\rep_\para^{-1}(\repx)\cdot p(\repx)\right|\\
        &=\sup_{(\para,\lat)\in\parasp\times\suppd}\|\rep_\para(\lat)\|^{n+3}\cdot |f(\lat)|\cdot \underbrace{\left|\det\Jac\rep_\para^{-1}(\rep_\para(\lat))\right|}_{\leq c}\cdot |p(\rep_\para(\lat))|\\
        &=c\cdot\sup_{\lat\in\suppd}q(\lat)\cdot |f(\lat)|<\infty
      \end{align*}
      by definition of Schwartz functions and the claim follows with \cref{lem:uniint}.
  \end{proof}

\begin{restatable}{lemma}{uniintder}
  \label{lem:uniintder}
  If $f:\suppd\to\Real$ is a Schwartz function, $\rep_{(-)}:\parasp\times\Real^n\to\Real^n$ satisfies \cref{ass:basic,eq:strdiff}, and $p:\Real^n\to\Real$ a polynomial then
  \begin{align*}
    \int_\suppd \sup_{\para\in\parasp}\left|\frac{\partial f_{(-)}}{\partial\theta_i}(\para,\repx)\cdot p(\repx) \right|\diff\repx&<\infty\\
    \int_\suppd \sup_{\para\in\parasp}\left|\frac{\partial f_{(-)}}{\partial z_i}(\para,\repx)\cdot p(\repx) \right|\diff\repx&<\infty\\
    \int_\suppd \sup_{\para\in\parasp}\left|\frac{\partial^2 f_{(-)}}{\partial\theta_i\partial\theta_j}(\para,\repx)\cdot p(\repx) \right|\diff\repx&<\infty&
  \end{align*}
  where $f_\para(\repx)\defeq f(\rep_{\para}^{-1}(\repx))\cdot |\det\Jac\rep_\para^{-1}(\repx)|$.
\end{restatable}
\begin{proof}
  As for \cref{lem:finmeas}, by \cref{lem:uniint} it suffices to prove
  \begin{align*}
    \sup_{(\para,\repx)\in\parasp\times\suppd}\|\repx\|^{n+2}\cdot \left|\frac{\partial f_{(-)}}{\partial\theta_i}\cdot p(\repx)\right|
  \end{align*}
  for the first claim.
    %  (by \cite[Lemma~2]{H58}).
    Note that
      \begin{align*}
        0=\Jac_{\theta_i}(\rep_\para\circ\rep_\para^{-1})(\repx)
        =\Jac_{\theta_i}\rep_\para(\rep_\para^{-1}(\repx))+\Jac_\lat \rep_\para(\rep_\para^{-1}(\repx))\cdot \Jac_{\theta_i}\rep_\para^{-1}(\repx)
      \end{align*}
      and hence,
      \begin{align*}
        \Jac_{\theta_i}\rep_\para^{-1}(\repx)&=-\left(\Jac_\lat \rep_\para(\rep_\para^{-1}(\repx))\right)^{-1}\cdot\Jac_{\theta_i}\rep_\para(\rep_\para^{-1}(\repx))\\
        &=-\frac 1{|\det\Jac\rep_\para(\rep_\para^{-1}(\repx))|}\cdot\mathrm{adj}(\Jac\rep_\para(\rep_\para^{-1}(\repx)))\cdot\Jac_{\theta_i}\rep_\para(\rep_\para^{-1}(\repx))\\
         &=-\frac {1}{|\det\Jac\rep_\para(\rep_\para^{-1}(\repx))|}\cdot \mathbf{g}(\para,\rep_\para^{-1}(\repx))
      \end{align*}
      for a suitable\footnote{$\mathrm{adj}$ is the adjugate matrix} function $\mathbf{g}:\parasp\times\Real^n\to\Real^n$ bounded by polynomials (component-wise).
      Therefore,
      \begin{align*}
        \frac{\partial(f\circ \rep_{(-)}^{-1})}{\partial\theta_i}(\repx)
        &=\nabla f(\rep_\para^{-1}(\repx))\cdot\Jac_{\theta_i}\rep_\para^{-1}(\repx)\\
        &=\frac {\nabla f(\rep_\para^{-1}(\repx))\cdot \mathbf{g}(\para,\rep_\para^{-1}(\repx))}{|\det\Jac\rep_\para(\rep_\para^{-1}(\repx))|}
      \end{align*}
      By \cref{eq:strdiff}, there exists $c>0$ satisfying
      \begin{align*}
        \frac 1{|\det\Jac\rep_\para(\rep_\para^{-1}(\repx))|}=|\det\Jac\rep^{-1}_\para(\repx)|&\leq c  
      \end{align*} for all $\para\in\parasp$ and $\repx\in\suppd$.
      Consequently,
      % \begin{align*}
      %   \left|\frac{\partial(f\circ \rep_{(-)}^{-1})}{\partial\theta_i}(\rep_\para(\lat))\right|
      % \end{align*}
      \begin{align*}
         &\sup_{(\para,\repx)\in\parasp\times\suppd}\left|\|\repx\|^{n+2}\cdot\frac{\partial(f\circ\rep^{-1}_{(-)})}{\partial\theta_i}(\repx)\cdot\det\Jac\rep^{-1}_\para(\repx)\cdot p(\repx) \right|\\
         &\leq c^2\cdot\sup_{(\para,\lat)\in\parasp\times\suppd}\left|\|\rep_\para(\lat)\|^{n+2}\cdot (\nabla f(\lat)\cdot \mathbf{g}(\para,\lat))\cdot p(\rep_\para(\lat)) \right|<\infty
      \end{align*}
      because $\rep_\para(\lat)$, $\mathbf{g}(\para,\lat)$ and $p(\rep_\para(\lat))$ are uniformly bounded by polynomials independent of $\para$ (by \cref{lem:polyp}) and derivatives of Schwartz functions are Schwartz functions, too.

      Likewise, note that
      % Recall that %$\det\Jac\rep_\para(\lat)$ is  and we can assume w.l.o.g.\ that $p_1$ is positive.
    \begin{align*}
      |\det\Jac\rep_\para^{-1}(\repx)|=\frac 1{|\det\Jac\rep_\para(\rep_\para^{-1}(\repx))|}=\frac 1 {h(\para,\rep_\para^{-1}(\repx))}
      % |\det\Jac\rep_\para^{-1}(\rep_\para(\lat))|=\frac 1{|\det\Jac\rep_\para(\lat)|}=\frac 1 {q(\para,\lat)}
    \end{align*}
    for a function $h:\parasp\times\Real^n\to\Real$, the partial derivatives of which are bounded by polynomials 
    and which we can assume by \cref{ass:basic} w.l.o.g.\ to be positive
     (and greater than the constant $c$ above).
  
    Thus,
      \begin{align*}
        \frac{\partial}{\partial\theta_i}\det\Jac\rep_\para^{-1}(\repx)&=
        \frac\partial{\partial\theta_i}\frac 1{h(\para,\rep_\para^{-1}(\repx))}\\
        &=-\frac{\frac{\partial h}{\partial\theta_i}(\para,\rep_\para^{-1}(\repx))+\nabla_\lat h(\para,\rep_\para^{-1}(\repx))\cdot\Jac_{\theta_i}\rep_\para^{-1}(\repx)}{\left(h(\para,\rep_\para^{-1}(\repx))\right)^{2}}\\
        &=-\frac{\frac{\partial h}{\partial\theta_i}(\para,\rep_\para^{-1}(\repx))}{\underbrace{\left(h(\para,\rep_\para^{-1}(\repx))\right)^{2}}_{|\det\Jac\rep^{-1}_\para(\repx)|^{-2}}}-\frac{\nabla_\lat h(\para,\rep_\para^{-1}(\repx))\cdot \mathbf{g}(\para,\rep_\para^{-1}(\repx))}{\underbrace{\left(h(\para,\rep_\para^{-1}(\repx))\right)^{3}}_{|\det\Jac\rep^{-1}_\para(\repx)|^{-3}}}
      \end{align*}
      to show 
      \begin{align*}
        \sup_{(\para,\repx)\in\parasp\times\suppd}\left|\|\repx\|^{n+2}\cdot f(\rep_\para^{-1}(\repx))\cdot\frac{\partial}{\partial\theta_i}\det\Jac\rep^{-1}_\para(\repx)\cdot p(\repx) \right|<\infty
      \end{align*}
      The same insights can be used to show the second and third bounds.
\end{proof}

\begin{restatable}[Lipschitz Smoothness]{corollary}{lipsch}
  \label{prop:lipsch}
  If $F\in\sfct$ then the function
  \begin{align*}
    \para\mapsto\E_{\lat\sim\mdist}[\sem F(\rep_\para(\lat))]
  \end{align*} is Lipschitz smooth.
\end{restatable}
In the same manner we can prove the other (simpler) obligations of \ref{it:bounded}.

\subsection{Supplementary Materials for \cref{sec:unicon}}
\label{app:unicon}

% \subsection{Supplementary Materials for \cref{sec:unicon}}

\unicon*
\begin{proof}
  Let $p:\Real^n\to\Real$ be a polynomial bound to $\sema F$ and $\sem F$ and let $\epsilon>0$. We focus on the second result (the first is analogous). We define
  \begin{align*}
    c&\defeq \int_\suppd\left|\sup_{\para\in\parasp}\frac{\partial\mdist_{(-)}}{\partial\theta_i}(\para,\repx)\right|\diff\repx\\
    U_k&\defeq\left\{\repx\in\Real^n\mid |f_\eta(\repx)-f(\repx)|>\frac\epsilon {2c}\text{ for some }0<\eta<\frac 1 k\right\}\\
    \mu(U)&\defeq\int_U p(\repx)\cdot\sup_{\para\in\parasp}\left|\frac{\partial\mdist_{(-)}}{\partial\theta_i}(\para,\repx)\right|\diff\repx
  \end{align*} 
  which is a finite measure by \cref{lem:uniintder}.
  Note that $(U_k)_{k\in\nat}$ is a non-increasing sequence of sets and $\bigcap_{k\in\nat} U_k$ is negligible.
  Hence, by continuity from above (of $\mu$) there exists $k$ such that $\mu(U_k)<\frac \epsilon 4$. 
  %  W.l.o.g.\ we can assume $\frac 1 k\leq\frac \epsilon 2$ (otherwise take a larger $k$).
  Finally, it suffices to observe that for $0<\eta<\frac 1 k$ and $\para\in\parasp$:
    \begin{align*}
      &\left|\frac\partial{\partial\theta_i}\E_{\lat\sim\mdist}[\sema F(\rep_\para(\lat))]-\frac\partial{\partial\theta_i}\E_{\lat\sim\mdist}[\sem F(\rep_\para(\lat))]\right|\\
      &=\left|\frac\partial{\partial\theta_i}\E_{\repx\sim\mdist_\para}[\sema F(\repx)]-\frac\partial{\partial\theta_i}\E_{\repx\sim\mdist_\para}[\sem F(\repx)]\right|\\
      &\leq \int_\suppd \left|\frac{\partial\mdist_{(-)}}{\partial\theta_i}(\para,\repx)\right|\cdot |\sema F(\repx)-\sem F(\repx)|\diff\repx\\
      &\leq \int_{U_k} \left|\frac{\partial\mdist_{(-)}}{\partial\theta_i}(\para,\repx)\right|\cdot |\sema F(\repx)-\sem F(\repx)|\diff\repx\\
      &\qquad+
      \int_{\suppd\setminus U_k} \left|\frac{\partial\mdist_{(-)}}{\partial\theta_i}(\para,\repx)\right|\cdot |\sema F(\repx)-\sem F(\repx)|\diff\repx\\
      &\leq 2\cdot\int_{U_k} \left|\frac{\partial\mdist_{(-)}}{\partial\theta_i}(\para,\repx)\right|\cdot p(\repx)\diff\repx+
      \int_{\suppd\setminus U_k} \left|\frac{\partial\mdist_{(-)}}{\partial\theta_i}(\para,\repx)\right|\cdot \frac\epsilon{2c}\diff\repx\\
      &\leq 2\cdot\mu(U_k)+
      \frac\epsilon{2c}\cdot\int_\suppd \left|\frac{\partial\mdist_{(-)}}{\partial\theta_i}(\para,\repx)\right|\diff\repx\\
      &\leq\epsilon\qedhere
    \end{align*}
\end{proof}

Uniform convergence may fail if $\parasp\subseteq\Real^m$ is not compact:
\begin{example}
  \label{ex:comp}
  Let $\parasp\defeq\Real\times\Real_{>0}$. $\E_{z\sim\nd(\theta_1,\theta_2)}[\smooth(z)]$ does \emph{not} converge  uniformly to $\E_{z\sim\nd(\theta_1,\theta_2)}[[z\geq 0]]$: Suppose $\eta>0$. There exists $\delta>0$ such that $\smooth(\delta)=0.6$. Define $\theta_1\defeq\theta_2\defeq\frac\delta 2$. Now, it suffices to note that
  \begin{align*}
    \E_{z\sim\nd(\theta_1,\theta_2)}[[z\geq 0]]&\geq\int_{\theta_1-\theta_2}^\infty\nd(z\mid\theta_1,\theta_2)\diff z\geq 0.83 \\
    \E_{z\sim\nd(\theta_1,\theta_2)}[\smooth(z)]&\leq 0.5\int_{-\infty}^{\theta_1-\theta_2}\nd(z\mid\theta_1,\theta_2)\diff z+
    0.6\int_{\theta_1-\theta_2}^{\theta_1+\theta_2}\nd(z\mid\theta_1,\theta_2)\diff z+
    \int_{\theta_1+\theta_2}^\infty\nd(z\mid\theta_1,\theta_2)\diff z\\
    &\leq 0.5\cdot 0.16+0.6\cdot 0.69+0.16= 0.654
  \end{align*}
% \dw{could plot something to illustrate}
\end{example}

\subsection{Supplementary Materials for \cref{sec:var}}
\label{app:var}

Now, exploiting the chain rule and \cref{ass:basic}, %and the fact that $|\smooth'|\leq\eta^{-1}$,
it is straightforward to show inductively that for $F\in\sfct_\ell$,
\begin{lemma}
  \label{lem:smsigder}
    If $F\in\sfct_\ell$ there exists a polynomial $p:\Real^n\to\Real$ such that for all $\eta>0$ and $\repx\in\Real^n$,
    \begin{align*}
        \left|\frac{\partial\sema F}{\partial z_i}(\repx)\right|\leq |\smooth'|^\ell\cdot p(\repx)
    \end{align*}
\end{lemma}
(By $|\smooth'|$ we mean $\sup_{y\in\Real}|\smooth'(y)|$.)

\begin{restatable}{lemma}{varaux}
  \label{lem:varaux}
  Let $F\in\sfct_\ell$, $f:\suppd\to\Real$ be a non-negative Schwartz function, where $\suppd=U_1\times\cdots\times U_n$ and $U_1,\ldots,U_n\in\{\Real,\Real_{\geq 0}\}$, and $p$ be a non-negative polynomial.
  For all $1\leq i\leq n$ there exists $c>0$ such that for all \changed[dw]{$0<\eta\leq 1$},
  \begin{align*}
    \int_\suppd f_\para(\repx)\cdot p(\repx)\cdot\left|\frac{\partial^2 \sema F}{\partial z_i^2}(\repx)\right|\diff\repx&<c\cdot\eta^{-\ell}\\
    \int_\suppd f_\para(\repx)\cdot p(\repx)\cdot \left|\frac{\partial \sema F}{\partial z_i}(\repx)\right|^2\diff\repx&<c\cdot\eta^{-\ell}
  \end{align*}
  where $f_\para\defeq f(\rep_\para^{-1}(\repx))\cdot|\det\Jac\rep_\para^{-1}(\repx)|$.
\end{restatable}
\begin{proof}
  Note that $f_\para$ is differentiable and non-negative. To simplify notation, we assume that $\suppd=\Real^n$. (Otherwise the proof is similar, exploiting \cref{lem:smsigder}.) Besides, it suffices to establish the first bound. To see that the second is a consequence of the first, we use integration by parts
  \begin{align*}
    &\int f_\para(\repx)\cdot p(\repx)\cdot \left(\frac{\partial \sema F}{\partial z_i}(\repx)\right)^2\diff\repx\\
    &=\int\int \left(f_\para(\repx)\cdot p(\repx)\cdot \frac{\partial \sema F}{\partial z_i}(\repx)\right)\cdot\frac{\partial \sema F}{\partial z_i}(\repx)\diff z_j\diff\repx_{-j}\\
    &=\int\underbrace{\left[f_\para(\repx)\cdot p(\repx)\cdot \frac{\partial \sema F}{\partial z_i}(\repx)\cdot \sema F(\repx)\right]^{\infty}_{-\infty}}_{=0}\diff\repx_{-j}\\
    &\qquad-\int \frac\partial{\partial z_i}\left(f_\para(\repx)\cdot p(\repx)\cdot \frac{\partial \sema F}{\partial z_i}(\repx)\right)\cdot \sema F(\repx)\diff\repx
  \end{align*}
  where $\repx_{-j}$ is $z_1,\ldots,z_{j-1},z_{j+1},\ldots,j_n$,
  because by \cref{lem:finmeas}, for fixed $\eta>0$ and $\repx_{-j}$,
  \begin{align*}
      \lim_{z_j\to\pm\infty}f_\para(\repx)\cdot \underbrace{p(\repx)\cdot \frac{\partial \sema F}{\partial z_i}(\repx)\cdot \sema F(\repx)}_{\text{bounded by polynomial}}
  \end{align*}
  % \cref{lem:smsigder,rem:unibsm}, , $z_j\mapsto p(\repx)\cdot \frac{\partial \sema F}{\partial z_i}(\repx)\cdot \sema F(\repx)$ is bounded by a polynomial, and .
We continue bounding:
  \begin{align*}
&-\int \frac\partial{\partial z_i}\left(f_\para(\repx)\cdot p(\repx)\cdot \frac{\partial \sema F}{\partial z_i}(\repx)\right)\cdot \sema F(\repx)\diff\repx\\
&=-\int \frac{\partial f_{(-)}}{\partial z_i}(\para,\repx)\cdot p(\repx)\cdot \frac{\partial \sema F}{\partial z_i}(\repx)\cdot \sema F(\repx)\diff\repx\\
&\qquad-\int f_\para(\repx)\cdot \frac{\partial p}{\partial z_i}(\repx)\cdot \frac{\partial \sema F}{\partial z_i}(\repx)\cdot \sema F(\repx)\diff\repx\\
&\qquad-\int f_\para(\repx)\cdot p(\repx)\cdot \frac{\partial^2 \sema F}{\partial z^2_i}(\repx)\cdot \sema F(\repx)\diff\repx\\
&\leq \eta^{-\ell}\cdot\int\sup_{\para\in\parasp}\left|\frac{\partial f_{(-)}}{\partial z_i}(\para,\repx)\cdot p_1(\repx)\right|\diff\repx\\
&\qquad+\eta^{-\ell}\cdot\int\sup_{\para\in\parasp}\left| f_\para(\repx)\cdot p_2(\repx)\right|\diff\repx\\
&\qquad+ \int f_\para(\repx)\cdot p_3(\repx)\cdot\left|\frac{\partial^2 \sema F}{\partial z^2_i}(\repx)\right|
  \end{align*}
  for suitable polynomial bounds $p_1,p_2,p_3:\Real^n\to\Real$ (which exist due to the fact that $\frac{\partial p}{\partial z_i}$ \emph{is} a polynomial, \cref{it:polymon,lem:smsigder}) and the second inequality follows with \cref{lem:uniintder}.

  We prove the claim by induction on the definition of $\sfct_\ell$:
  \begin{itemize}
      \item For $z_j\in\sfct_0$ and $F\in\sfct_{\ell+1}$ due to $F\in\sfct_\ell$ the claim is obvious.
      \item For $g(F_1,\ldots, F_k)\in\sfct_\ell$ because $F_1,\ldots,f_k\in\sfct_\ell$,
      \begin{align*}
          &\int f_\para(\repx)\cdot p(\repx)\cdot \left|\frac{\partial^2 (g\circ\langle \sema{F_1},\ldots,\sema{F_k}\rangle)}{\partial z_i^2}(\repx)\right|\diff\repx\\
          &\leq\sum_{1\leq j,j'\leq k}\int f_\para(\repx)\cdot p(\repx)\cdot \left|\frac{\partial^2 g}{\partial y_j\partial y_{j'}}(\sema{F_1}(\repx),\ldots, \sema{F_k}(\repx))\cdot\frac{\partial \sema {F_j}}{\partial z_i}(\repx)\cdot\frac{\partial \sema {F_{j'}}}{\partial z_i}(\repx)\right|\diff\repx\\
           &\qquad+\sum_{1\leq j\leq k}\int f_\para(\repx)\cdot p(\repx)\cdot \left|\frac{\partial g}{\partial y_j}(\sema{F_1}(\repx),\ldots, \sema{F_k}(\repx))\cdot\frac{\partial^2 \sema {F_j}}{\partial^2 z_i}(\repx)\right|\diff\repx
        \end{align*}
        By \cref{ass:basic}, $\frac{\partial g}{\partial y_j}(g_\eta^{(1)}(\repx),\ldots,g_\eta^{(\ell)}(\repx))$ is bounded by a polynomial. Therefore the second summand can be bounded by the inductive hypothesis.
        For the first, again by \cref{ass:basic}, we can bound
        \begin{align*}
            p(\repx)\cdot\left|\frac{\partial^2 g}{\partial y_j\partial y_{j'}}(g_\eta^{(1)}(\repx),\ldots,g_\eta^{(\ell)}(\repx))\right|\leq p_1(\repx)
        \end{align*} 
        for a (non-negative) polynomial $p_1$, apply the Cauchy-Schwarz inequality and apply the inductive hypothesis to
        \begin{align*}
        %   &\E_{\repx\sim\mdist_\para}\left[p(\repx)\cdot \left|\frac{\partial^2 g}{\partial y_j\partial y_{j'}}(\sema{F_1}(\repx),\ldots, \sema{F_k}(\repx))\cdot\frac{\partial \sema {F_j}}{\partial z_i}(\repx)\cdot\frac{\partial \sema {F_{j'}}}{\partial z_i}(\repx)\right|\right]\\
          &\int f_\para(\repx)\cdot p_1(\repx)\cdot \left|\frac{\partial \sema {F_j}}{\partial z_i}(\repx)\right|^2\diff\repx&
          &\int f_\para(\repx)\cdot p_1(\repx)\cdot\left|\frac{\partial \sema {F_{j'}}}{\partial z_i}(\repx)\right|^2\diff\repx
        \end{align*}
        \item Next, suppose $\ifc{F_1}{F_2}{F_3}\in\sfct_{\ell+1}$ because $F_1\in\sfct_\ell$ and $F_2,F_3\in\sfct_{\ell+1}$. By linearity we bound (similarly for the other branch):
        \begin{align*}
          &\int f_\para(\repx)\cdot p(\repx)\cdot \left|\frac{\partial^2 ((\smooth\circ\sema {F_1})\cdot\sema {F_3})}{\partial z^2_i}(\repx)\right|\diff\repx\\
          &\leq\int f_\para(\repx)\cdot p(\repx)\cdot \left|\frac{\partial^2 (\smooth\circ\sema {F_1})}{\partial z_i^2}(\repx)\cdot\sema {F_3}(\repx)\right|\diff\repx\\
           &\qquad+2\cdot\int f_\para(\repx)\cdot p(\repx)\cdot \left|\frac{\partial (\smooth\circ\sema {F_1})}{\partial z_i}(\repx)\cdot\frac{\partial\sema {F_3}}{\partial z_i}(\repx)\right|\diff\repx\\
           &\qquad+\int f_\para(\repx)\cdot p(\repx)\cdot \left|\smooth(\sema {F_1}(\repx))\cdot\frac{\partial^2\sema {F_3}}{\partial z_i^2}(\repx)\right|\diff\repx
        \end{align*}
          We abbreviate $F\equiv F_1$ and bound $p(\repx)\cdot|\sema{F_3}(\repx)|$ by the non-negative polynomial $p_2$. Bounding the first summand is most interesting:
        \begin{align*}
          &\int f_\para(\repx)\cdot p(\repx)\cdot \left|\frac{\partial^2 (\smooth\circ\sema {F_1})}{\partial z_i^2}(\repx)\cdot\sema {F_3}(\repx)\right|\diff\repx\\
          &=\int f_\para(\repx)\cdot p_2(\repx)\cdot \left|\frac{\partial^2 (\smooth\circ \sema F)}{\partial z_i}(\repx)\right|\diff\repx\\
          &\leq\int f_\para(\repx)\cdot p_2(\repx)\cdot \left|\smooth''(\sema F(\repx))\right|\cdot\left(\frac{\partial \sema F}{\partial z_i}(\repx)\right)^2\diff\repx\\
           &\qquad+\int f_\para(\repx)\cdot p_2(\repx)\cdot \left|\underbrace{\smooth'(\sema F(\repx))}_{\leq\eta^{-1}}\cdot\frac{\partial^2 \sema F}{\partial z_i^2}(\repx)\right| \diff\repx
        \end{align*}
        The first summand can be bounded with the inductive hypothesis.
        For the second summand we exploit that
        \begin{align*}
            |\smooth''(x)|&=\eta^{-2}\left|\smooth(x)\cdot(1-\smooth(x))\cdot(1-2\smooth(x))\right|\\
          &\leq\eta^{-2}\cdot\smooth(x)\cdot(1-\smooth(x))\\
          &= \eta^{-1}\cdot\smooth'(x)
        \end{align*}
        and continue using integration by parts again in the second step
        \begin{align*}
          &\int f_\para(\repx)\cdot p_2(\repx)\cdot \left|\smooth''(\sema F(\repx))\cdot\left(\frac{\partial \sema F}{\partial z_i}(\repx)\right)^2 \right|\diff\repx\\
          &\leq\eta^{-1}\cdot\int f_\para(\repx)\cdot p_2(\repx)\cdot\frac{\partial \sema F}{\partial z_i}(\repx)\cdot\frac{\partial(\smooth\circ \sema F)}{\partial z_i}(\repx)\diff\repx\\
          &=-\eta^{-1}\cdot\int\frac\partial{\partial z_i}\left( f_\para(\repx)\cdot p_2(\repx)\cdot\frac{\partial \sema F}{\partial z_i}(\repx)\right)\cdot\smooth(\sema F(\repx))\diff\repx\\
          &=-\eta^{-1}\cdot\int\frac{\partial f_{(-)}}{\partial z_i}(\para,\repx)\cdot p_2(\repx)\cdot\frac{\partial \sema F}{\partial z_i}(\repx)\cdot\smooth(\sema F(\repx))\diff\repx\\
           &\qquad-\eta^{-1}\cdot\int f_\para(\repx)\cdot \frac{\partial p_2}{\partial z_i}(\repx)\cdot\frac{\partial \sema F}{\partial z_i}(\repx)\cdot\smooth(\sema F(\repx))\diff\repx\\
           &\qquad-\eta^{-1}\cdot\int f_\para(\repx)\cdot p_2(\repx)\cdot\frac{\partial^2 \sema F}{\partial z^2_i}(\repx)\cdot\smooth(\sema F(\repx))\diff\repx\\
          &\leq \eta^{-1}\cdot\int\left|\frac{\partial f_{(-)}}{\partial z_i}(\para,\repx)\cdot p_2(\repx)\cdot\frac{\partial \sema F}{\partial z_i}(\repx)\right|\diff\repx\\
           &\qquad+\eta^{-1}\cdot\int f_\para(\repx)\cdot\left|\frac{\partial p_2}{\partial z_i}(\repx)\cdot\frac{\partial \sema F}{\partial z_i}(\repx)\right|\diff\repx\\
           &\qquad+\eta^{-1}\cdot\int f_\para(\repx)\cdot\left|p_2(\repx)\cdot\frac{\partial^2 \sema F}{\partial z^2_i}(\repx)\right|\diff\repx
        \end{align*}
        and the claim follows with the inductive hypothesis (recall $F\equiv F_1\in\sfct_\ell$), \cref{lem:smsigder,lem:uniintder}.\qedhere
  \end{itemize}
\end{proof}

\var*
\begin{proof}
  Note that
  \begin{align*}
    \E_{\lat\sim\mdist}\left[\left|\frac{\partial(\sema F\circ\rep_{(-)})}{\partial \theta_i}(\lat)\right|^2\right]
    &=
    \E_{\lat\sim\mdist}\left[\left|\sum_{j=1}^n \frac{\partial\sema F}{\partial z_j}(\rep_\para(\repx))\cdot\frac{\partial\phi^{(j)}_{(-)}}{\partial\theta_i}(\para,\lat)\right|^2\right]
    % &\leq\sum_{1\leq j,j'\leq n}\sqrt{\E_{\lat\sim\mdist}\left[\left|\frac{\partial\sema F}{\partial z_j}(\rep_\para(\repx))\cdot\frac{\partial\phi^{(j)}_{(-)}}{\partial\theta_i}(\para,\lat)\right|^2\right]\cdot
    % \E_{\lat\sim\mdist}\left[\left|\frac{\partial\sema F}{\partial z_{j'}}(\rep_\para(\repx))\cdot\frac{\partial\phi^{(j')}_{(-)}}{\partial\theta_i}(\para,\lat)\right|^2\right]}
  \end{align*}
  Therefore, by the Cauchy-Schwarz inequality it suffices to bound
  \begin{align*}
    \E_{\lat\sim\mdist}\left[\left|\frac{\partial\sema F}{\partial z_j}(\rep_\para(\repx))\cdot\frac{\partial\phi^{(j)}_{(-)}}{\partial\theta_i}(\para,\lat)\right|^2\right]
  \end{align*}
  for all $1\leq j\leq n$.
%   In particular, it suffices to bound each of the expectations.
  By \cref{lem:polyp}, for each $1\leq i\leq m$ and $1\leq j\leq n$ there exists a (non-negative) polynomial bound $p:\Real^n\to\Real$  satisfying
  \begin{align*}
    \left(\frac{\partial\phi^{(j)}_{(-)}}{\partial\theta_i}(\para,\lat)\right)^2\leq p(\lat)
  \end{align*}
  and for $f(\lat)\defeq\mdist(\lat)\cdot p(\lat)$, which is a Schwartz function by \cref{it:prod},
  \begin{align*}
    \E_{\lat\sim\mdist}\left[\left|\frac{\partial\sema F}{\partial z_j}(\rep_\para(\repx))\cdot\frac{\partial\phi^{(j)}_{(-)}}{\partial\theta_i}(\para,\lat)\right|^2\right]&\leq
    \int f(\lat)\cdot \left|\frac{\partial\sema F}{\partial z_j}(\rep_\para(\repx))\right|^2\diff\lat\\
    &=\int f_\para(\repx)\cdot \left|\frac{\partial\sema F}{\partial z_j}(\repx)\right|^2\diff\repx
  \end{align*}
  where
  \begin{align*}
    f_\para\defeq f(\rep_\para^{-1}(\repx))\cdot|\det\Jac\rep_\para^{-1}(\repx)|
  \end{align*}
  The claim follows with \cref{lem:varaux}.
\end{proof}

\subsection{Supplementary Materials for \cref{sec:cc}}
\begin{restatable}[Average Variance of Run]{remark}{avvar}
  % \label{rem:avvar}
  By \cref{prop:var} the average variance of a finite DSGD run with length $N$ is bounded by (using H\"older's inequality)
  \[
  \frac 1N\cdot\sum_{k=1}^N\eta_k^{-\ell} \leq \eta_{\frac{N+1}2}^{-\ell}
  \]
  Consequently, the average variance of a DSGD run is lower than for standard SGD with a fixed accuracy coefficient $\eta<\eta_{\frac{N+1}2}$, the accuracy coefficient of DSGD after half the iterations.
\end{restatable}
\begin{proof}
  If $\eta_k\in\Theta(k^{-\frac 1\ell+\epsilon})$ then (modulo constants), $\eta_k^{-\ell}=k^\delta$ for $\delta\leq 1$. Therefore, by H\"older's inequality,
  \begin{align*}
    \frac 1N\cdot\sum_{k=1}^N\eta_k^{-\ell}&=
    \frac 1N\cdot\sum_{k=1}^N k^\delta\\
    &\leq\frac 1 N\cdot \left(N^{1/\delta-1}\cdot\sum_{k=1}^N k\right)^\delta\\
&=\frac 1 N\cdot\left(N^{1/\delta-1}\cdot\frac{N\cdot(N+1)}2\right)^\delta\\
&=\left(\frac{N+1}2\right)^\delta
=\eta_{\frac{N+1}2}^{-\ell}\qedhere
  \end{align*}
\end{proof}

\section{Supplementary Materials for \cref{sec:eval}}
\label{app:eval}

The code is available at \url{https://github.com/domwagner/DSGD.git}.
   The experiements can be viewed and run in the jupyter notebook \texttt{experiments.ipynb} by running:

\texttt{jupyter notebook experiments.ipynb}.

\subsection{Additional Models}

\begin{itemize}
  \item \texttt{cheating} %, adapted from the book \emph{Bayesian Methods for Hackers}
  \citep{DavidsonPilon15} simulates a differential privacy setting where students taking an exam are surveyed to determine the prevalence of cheating without exposing the details for any individual.
  Students are tasked to toss a coin, on heads they tell the truth (cheating or not cheating) and on tails they toss a second coin to determine their answer.
  The tossing of coins here is a source of discontinuity.
  The goal, given the proportion of students who answered yes, is to predict a posterior on the cheating rate.
  In this model there are 300 if-statements and a 301-dimensional latent space, although we only optimise over a single dimension with the other 300 being sources of randomness.
  \item \texttt{textmsg} \citep{DavidsonPilon15} models daily text message rates, and the goal is to discover a change in the rate over the 74-day period of data given.
  The non-differentiability arises from the point at which the rate is modelled
  to change. The model has a 3-dimensional latent variable (the two rates and
  the point at which they change) and 37 if-statements.

  \item \texttt{influenza}  \citep{ShumwayS05} models the US influenza mortality for
  1969.
  In each month, the mortality rate depends on the dominant virus strain being of type 1 or type 2, producing a non-differentiablity for each month.
  Given the mortality data, the goal is to infer the dominant virus strain in each month. The model has a 37-dimensional latent variable and 24 if-statements.
\end{itemize}

% \subsection{Experimental Set-Up}
% \label{app:expsetup}

\subsection{Additional Results}

See \cref{fig:addtraj,tab:addvar,tab:finalelbo2}.

\begin{figure}[h]
  \begin{center}
    \begin{subfigure}[h]{0.45\linewidth}
      \includegraphics[width=\linewidth]{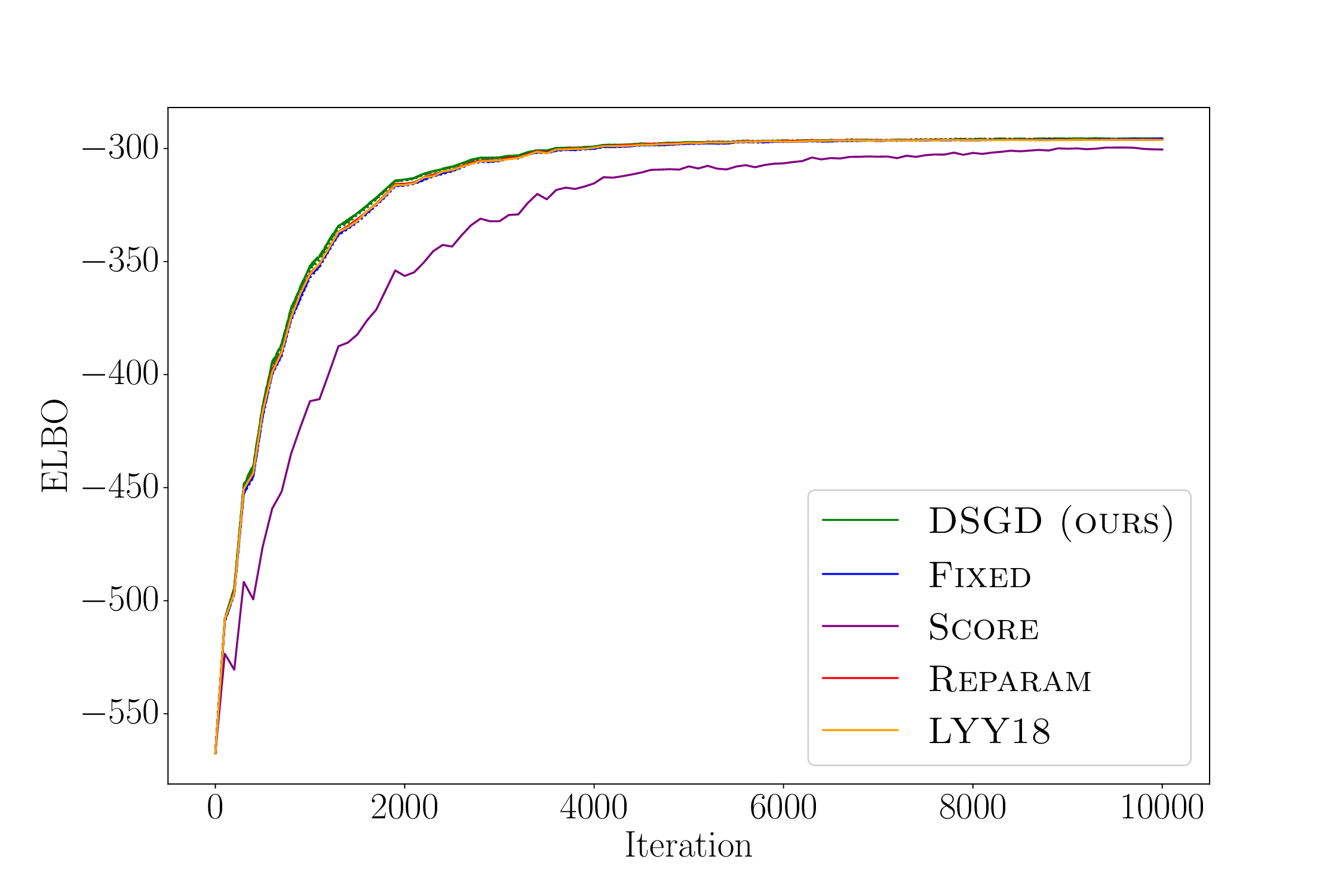}
    
      \vspace*{-3mm}
      \caption{\texttt{textmsg}}
      % \label{fig:cheating-graph}
    \end{subfigure}
    \begin{subfigure}[h]{0.45\linewidth}
      \includegraphics[width=\linewidth]{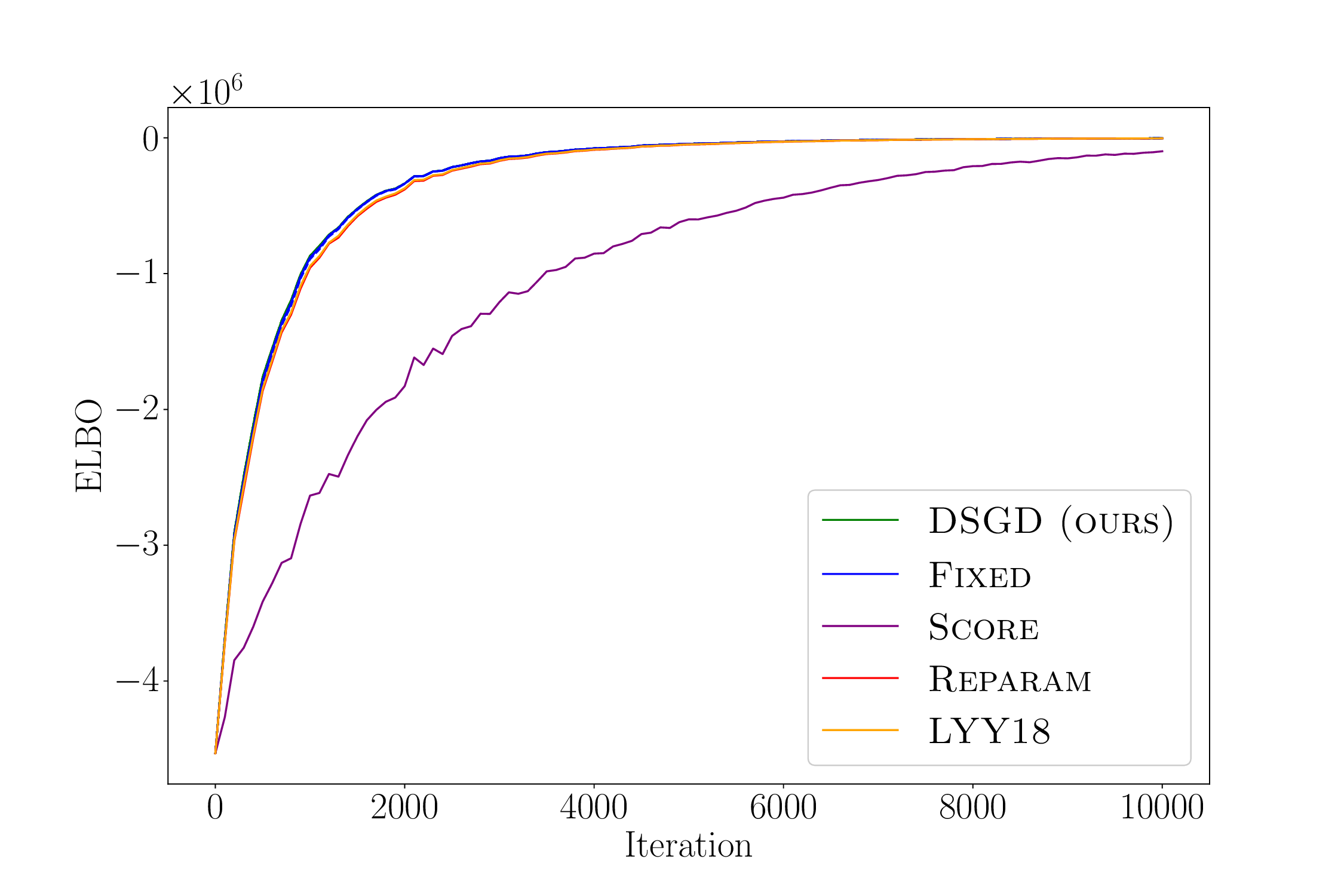}
    
      \vspace*{-3mm}
      \caption{\texttt{influenza}}
      % \label{fig:cheating-graph}
    \end{subfigure}
  \end{center}
  \caption{ELBO trajectories for additional models. The same setup is used as in \cref{fig:elbo}.}
  \label{fig:addtraj}
\end{figure}

\begin{table}
  \caption{Computational cost and work-normalised variances for additional models (cf.~\cref{tab:var})}
  \label{tab:addvar}
  \vspace*{2pt}
  \begin{subtable}{0.48\linewidth}
    \caption{\texttt{textmsg}}
    \centering
    \begin{tabular}{lcccr}
      \toprule
      Estimator & Cost & $\textrm{Avg}(V(.))$ & $V(\|.\|_2)$ \\
      \midrule
      \textsc{DSGD} (ours) &  1.84  &  7.89e-03   & 1.53e-02\\
      \textsc{Fixed} &   1.79 &   1.08e-02  &  2.14e-02\\
      \textsc{Reparam} &   1.25 &  7.99e-03  &  1.53e-02\\
      \textsc{LYY18} &    4.51 &   3.42e-02 &   6.00e-02 \\
      \bottomrule
    \end{tabular}
  \end{subtable}
\qquad
  % \vspace*{5pt}
  \begin{subtable}{0.48\linewidth}
    \caption{\texttt{influenza}}
    \centering
    \begin{tabular}{lcccr}
      \toprule
      Estimator & Cost & $\textrm{Avg}(V(.))$ & $V(\|.\|_2)$ \\
      \midrule
      \textsc{DSGD} (ours)&  1.28 &   7.77e-03 &   3.94e-03\\
      \textsc{Fixed} &    1.28 &   7.92e-03 &   3.97e-03\\
      \textsc{Reparam} &  1.21 &  7.60e-03 &   3.75e-03\\
      \textsc{LYY18} &  8.30 &  5.80e-02  &  2.88e-02\\
      \bottomrule
    \end{tabular}
  \end{subtable}

\end{table}

\begin{table*}[t]
  \caption{Mean of the final ELBO (the higher the better)
  for different random seeds and indicating error bars (the $\pm$ is one standard deviation).}
  \label{tab:finalelbo2}
  \vspace*{2pt}
  \begin{subtable}{\linewidth}
    \caption{\texttt{temperature}}
    \centering
  \begin{tabular}{c|ccccc}
    \toprule
    $\eta/\eta_{4000}$ & \textsc{DSGD} (ours) & \textsc{Fixed} & \textsc{Score} & \textsc{Reparam} & \textsc{LYY18}\\
    \midrule
    0.06 & -76 $\pm$ 1 & -624,250 $\pm$ 44,121\\
    0.1 & -84 $\pm$ 2 & -425 $\pm$ 9\\
    0.14 & -15,476 $\pm$ 4,641 & -121,932 $\pm$ 85,460 & -2,611,479 $\pm$ 255,193 & -706,729 $\pm$ 4,697 & -17,502 $\pm$ 52,044\\
    0.18 & -94,125 $\pm$ 6,930 & -32,171 $\pm$ 66\\
    0.22 & -165,787 $\pm$ 9,758 & -155,321 $\pm$ 61,732\\
    % 0.06 & -76 $\pm$ 1 & -624,250 $\pm$ 44,121\\
    % 0.1 & -84 $\pm$ 2 & -425 $\pm$ 9 & -2,611,479 $\pm$ 255,193 & -706,729 $\pm$ 4,697 & -17,502 $\pm$ 52,044\\
    % 0.14 & -15,476 $\pm$ 4,640 & -121,932 $\pm$ 85,460\\
    \bottomrule
  \end{tabular}
\end{subtable}

\vspace*{5pt}
\begin{subtable}{\linewidth}
  \caption{\texttt{xornet}}
  \centering
  \begin{tabular}{c|ccccc}
    \toprule
    $\eta/\eta_{4000}$ & \textsc{DSGD} (ours) & \textsc{Fixed} & \textsc{Score} & \textsc{Reparam}\\
    \midrule
    0.06 & -3,530 $\pm$ 3,889 & -5,522 $\pm$ 4,136\\
    0.1 & -33 $\pm$ 7 & -2,029 $\pm$ 3,305\\
    0.14 & -27 $\pm$ 4 & -2,028 $\pm$ 3,986 & -553 $\pm$ 1,507 & -9,984 $\pm$ 38\\
    0.18 & -25 $\pm$ 3 & -26 $\pm$ 6\\
    0.22 & -30 $\pm$ 8 & -26 $\pm$ 2 \\
    \bottomrule
  \end{tabular}
\end{subtable}

% etas = [0.14, 0.06, 0.1, 0.18, 0.22]
\vspace*{5pt}
\begin{subtable}{\linewidth}
  \caption{\texttt{random-walk}}
  \centering
\begin{tabular}{c|ccccc}
  \toprule
  $\eta/\eta_{4000}$ & \textsc{DSGD} (ours) & \textsc{Fixed} & \textsc{Score} & \textsc{Reparam} & \textsc{LYY18}\\
  \midrule
0.06 & -37 $\pm$ 148 & -85 $\pm$ 197\\
0.1 & -37 $\pm$ 148 & -86 $\pm$ 197\\
0.14 & -38 $\pm$ 148 & -37 $\pm$ 148 & -85 $\pm$ 197 & -371,612 $\pm$ 8,858 & -135 $\pm$ 226\\
0.18 & -38 $\pm$ 148 & -37 $\pm$ 148\\
0.22 & -38 $\pm$ 148 & -37 $\pm$ 148\\
  \bottomrule
\end{tabular}
\end{subtable}

\vspace*{5pt}
\begin{subtable}{\linewidth}
  \caption{\texttt{cheating}}
  \centering
\begin{tabular}{c|ccccc}
  \toprule
  $\eta/\eta_{4000}$ & \textsc{DSGD} (ours) & \textsc{Fixed} & \textsc{Score} & \textsc{Reparam} & \textsc{LYY18}\\
  \midrule
  0.06 & -65 $\pm$ 1 & -65 $\pm$ 1\\
  0.1 & -65 $\pm$ 1 & -65 $\pm$ 1\\
  0.14 & -65 $\pm$ 1 & -65 $\pm$ 1 & -66 $\pm$ 1 & -80 $\pm$ 1 & -65 $\pm$ 1\\
  0.18 & -65 $\pm$ 1 & -65 $\pm$ 1\\
  0.22 & -65 $\pm$ 1 & -65 $\pm$ 1\\
  \bottomrule
\end{tabular}
\end{subtable}

\vspace*{5pt}
\begin{subtable}{\linewidth}
  \caption{\texttt{textmsg}}
  \centering
\begin{tabular}{c|ccccc}
  \toprule
  $\eta/\eta_{4000}$ & \textsc{DSGD} (ours) & \textsc{Fixed} & \textsc{Score} & \textsc{Reparam} & \textsc{LYY18}\\
  \midrule
0.06 & -295  $\pm$  1 & -295  $\pm$  1\\
0.1 & -295  $\pm$  1 & -296  $\pm$  1\\
0.14 & -295  $\pm$  1 & -296  $\pm$  1 & -300  $\pm$  1 & -296  $\pm$  1 & -296  $\pm$  1\\
0.18 & -296  $\pm$  1 & -296  $\pm$  1\\
0.22 & -296  $\pm$  1 & -296  $\pm$  1\\
  \bottomrule
\end{tabular}
\end{subtable}

\vspace*{5pt}
\begin{subtable}{\linewidth}
  \caption{\texttt{influenza}}
  \centering
\begin{tabular}{c|ccccc}
  \toprule
  $\eta/\eta_{4000}$ & \textsc{DSGD} (ours) & \textsc{Fixed} & \textsc{Score} & \textsc{Reparam} & \textsc{LYY18}\\
  \midrule
  0.06 & -3,586  $\pm$  112 & -3,589  $\pm$  114\\
  0.1 & -3,584  $\pm$  111 & -3,590  $\pm$  111\\
  0.14 & -3,582  $\pm$  111 & -3,590  $\pm$  111 & -95,380  $\pm$  3,567 & -4,045  $\pm$  108 & -3,860  $\pm$  106\\
  0.18 & -3,579  $\pm$  111 & -3,589  $\pm$  111\\
  0.22 & -3,577  $\pm$  111 & -3,587  $\pm$  112\\
  \bottomrule
\end{tabular}
\end{subtable}
\end{table*}

\iffalse
\section{Connection to Probabilistic Programming}
\label{sec:prob}
\dw{REMOVE?}

Probabilistic programming is a programming paradigm which has the vision to make statistical methods, in particular Bayesian inference, accessible to a wide audience. This is achieved by a separation of concerns: the domain experts wishing to gain statistical insights focus on modelling, whilst the inference is performed automatically. \dw{reference some paper}

In essence, probabilistic programming languages extend more traditional programming languages with statistical constructs such as $\mathbf{observe}$ (as well as $\sample$) to define the prior $p(\repx)$ and likelihood $p(\mathbf x | \repx)$.

\emph{Variational inference} is a very successful approach to Bayesian inference, which frames posterior inference as an optimisation problem. In its simplest case we posit a variational family $\{q_\para\mid\para\in\parasp\}$ of simple distributions and we seek to find the member which is ``closest'' (as formalised by the KL-divergence) to the posterior $p(\repx\mid\mathbf x)$. Equivilantly, we seek to maximise the ELBO-function:
\begin{align*}
  \ELBO(\para)\defeq\E_{\repx\sim q_\para}[\log p(\repx,\mathbf x)-\log q_\para(\repx)]
\end{align*}

\subsection{An Idealised Probabilistic Programming Language}
We fix \emph{latent variables} $z_1,\ldots,z_n$ and \emph{program variables} typically denoted by $x$. \dw{or $v_1,\ldots$ better?} % and $z_i$ are the latent variables w.r.t. which the expectation is taken.
\emph{(Extended) expressions} are defined inductively as
\begin{align*}
  E&::= z_j\mid x\mid f(E,\ldots,E)\mid\ifc EEE
  % P&::= \skipp\mid x:= E;P\\
  % &\mid x:=\ifc EEE;P\\
  % &\mid z_j\sim\dist(E,\ldots,E);P\\
  % &\mid\observe E {\dist(E,\ldots,E)};P
\end{align*}
where $f\in\pop$.
Note that extended expressions without program variables (but with latent variables) are also in $\sfct$.
\dw{benefits: variable assignment}

A \emph{command} is one of the following
\begin{itemize}
  \item deterministic assignment $x:= E$
  \item sampling command: $z_j\sim\dist(E,\ldots,E)$
  \item observation: $\observe E {\dist(E,\ldots,E)}$
\end{itemize}
Finally, a \emph{program} is a sequence of commands.
We assume that for each latent variable $z_j$ there is exactly one corresponding sample command and program variables are assigned a value before they are used.
\begin{example}
  \label{ex:prog}
  The example of \cite[Prop.~2]{LYY18} exhibiting the bias of the reparameterisation gradient can be expressed as:
  \begin{align*}
    &z\sim\nd(0,1);\\
    &\mu:=\ifc z {-2}{5};\\
    &\observe 0{\nd(\mu,1)};
  \end{align*}
\end{example}

\subsection{Expressing Densities in Simple Function Calculus}
\newcommand\logp{\mathrm{logp}}
A program defines a log-density $\log p(z_1,\ldots,z_j)$, which is the sum of the log-pdfs accummulated by the sample and observe commands.
\dw{ELBO}
% We can transform a probabilistic program to a deterministic one compute the log-pdf by
% introducing a fresh variable $\logp$ which we initialise as $\logp:=0$ and eliminating each sampling and observe commands as follows:
% \begin{align*}
% z_j\sim\dist(E_1,\ldots,E_k)&\compi \logp:=\logp+\logpdf_{\dist}(z_j\mid E_1,\ldots, E_k)\\
% \observe E {\dist(E_1,\ldots,E_k)}&\compi \logp:=\logp+\logpdf_{\dist}(E\mid E_1,\ldots, E_k)
% \end{align*}

\dw{don't want to do this in practice, just for theoretical purposes}

The log-density of a program $P$ can be directly expressed as a term $F_P$ in our simple function calculus of \cref{sec:setup}.
\begin{example}
  The log-density of \cref{ex:prog} can be expressed as
  \begin{align*}
    \logpdf_\nd(z\mid 0,1)+\logpdf_\nd(4\mid\ifc z 2 5,1)
  \end{align*}
\end{example}
In general we first need to eliminate all deterministic assignments $x:=E$ by substituting $E$ for $x$ in the subsequent program. \dw{re-assignment} The remaining program does not contain program variables any more. It is a sequence of sample and observe commands, which we replace by a sum of log-densities:  $z_j\sim\dist(F_1,\ldots,F_k)$ is replaced with $\logpdf_{\dist}(z_j\mid F_1,\ldots, F_k)$ and $\observe F {\dist(F_1,\ldots,F_k)}$ is replaced with $\logpdf_{\dist}(F\mid F_1,\ldots, F_k)$.

\fi

\end{document}